%% file: main.tex
\documentclass{article}
\usepackage{arxiv_preprint,times}
\input{math_commands.tex}

\usepackage{xcolor}
\usepackage{hyperref}
\usepackage{url}
\usepackage{amsmath,amssymb,amsfonts,amsthm,mathtools}
\usepackage{algorithm}
\usepackage{algorithmic}
\usepackage{booktabs}
\usepackage{graphicx}
\usepackage{adjustbox}
\usepackage{float}
\usepackage{placeins}
\usepackage{multicol}
\usepackage{microtype}
\usepackage[most]{tcolorbox}
\definecolor{citeviolet}{RGB}{92,45,145}
\definecolor{linkblue}{RGB}{24,73,125}
\definecolor{urlteal}{RGB}{0,105,105}
\hypersetup{
  colorlinks=true,
  citecolor=citeviolet,
  linkcolor=linkblue,
  urlcolor=urlteal,
  hypertexnames=false,
  pdftitle={Density-Level Backward Error for Low-NFE Generative Flows},
  pdfauthor={Ewsey Obzherin; Valentin Khrulkov},
  pdfsubject={Generative flows, density-level backward error, numerical integration}
}

\title{Density-Level Backward Error for Low-NFE Generative Flows}

\author{%
  \textbf{Ewsey Obzherin}$^{1,2}$ \qquad
  \textbf{Valentin Khrulkov}$^{1}$\\[0.60em]
  \normalsize $^{1}$FusionBrain Lab\\[0.15em]
  \normalsize $^{2}$MIRIAI
}

\newtheorem{theorem}{Theorem}
\newtheorem{proposition}{Proposition}

\newtheorem{corollary}{Corollary}
\newtheorem{definition}{Definition}
\newtheorem{remark}{Remark}

\definecolor{thmblue}{RGB}{31,73,125}
\definecolor{propteal}{RGB}{28,99,93}
\definecolor{defgray}{RGB}{72,72,72}
\definecolor{remorange}{RGB}{150,88,20}
\tcolorboxenvironment{theorem}{enhanced,breakable,colback=blue!2!white,colframe=thmblue,boxrule=0.65pt,arc=1.2mm,left=1.2mm,right=1.2mm,top=1mm,bottom=1mm,before skip=7pt,after skip=7pt}
\tcolorboxenvironment{proposition}{enhanced,breakable,colback=teal!2!white,colframe=propteal,boxrule=0.55pt,arc=1.2mm,left=1.2mm,right=1.2mm,top=1mm,bottom=1mm,before skip=6pt,after skip=6pt}
\tcolorboxenvironment{lemma}{enhanced,breakable,colback=teal!2!white,colframe=propteal,boxrule=0.55pt,arc=1.2mm,left=1.2mm,right=1.2mm,top=1mm,bottom=1mm,before skip=6pt,after skip=6pt}
\tcolorboxenvironment{corollary}{enhanced,breakable,colback=green!2!white,colframe=green!35!black,boxrule=0.55pt,arc=1.2mm,left=1.2mm,right=1.2mm,top=1mm,bottom=1mm,before skip=6pt,after skip=6pt}
\tcolorboxenvironment{definition}{enhanced,breakable,colback=gray!4!white,colframe=defgray,boxrule=0.55pt,arc=1.2mm,left=1.2mm,right=1.2mm,top=1mm,bottom=1mm,before skip=6pt,after skip=6pt}
\tcolorboxenvironment{remark}{enhanced,breakable,colback=orange!4!white,colframe=remorange,boxrule=0.5pt,arc=1.2mm,left=1.2mm,right=1.2mm,top=1mm,bottom=1mm,before skip=6pt,after skip=6pt}
\tcolorboxenvironment{example}{enhanced,breakable,colback=purple!3!white,colframe=purple!50!black,boxrule=0.5pt,arc=1.2mm,left=1.2mm,right=1.2mm,top=1mm,bottom=1mm,before skip=6pt,after skip=6pt}
\newtcolorbox{keybox}{enhanced,breakable,colback=blue!3!white,colframe=thmblue,boxrule=0.7pt,arc=1.2mm,left=1.5mm,right=1.5mm,top=1mm,bottom=1mm}
\newtcolorbox{takeawaybox}{enhanced,breakable,colback=green!3!white,colframe=green!45!black,boxrule=0.7pt,arc=1.2mm,left=1.5mm,right=1.5mm,top=1mm,bottom=1mm}

\providecommand{\R}{\mathbb{R}}
\providecommand{\E}{\mathbb{E}}
\newcommand{\dd}{\mathrm{d}}

\newcommand{\diag}{\operatorname{diag}}
\newcommand{\cU}{\mathcal{U}}
\newcommand{\cG}{\mathcal{G}}
\newcommand{\cB}{\mathcal{B}}

\begin{document}
\maketitle

\input{sections/00_abstract}

\newcommand{\Rdens}{\mathfrak R}

\input{sections/01_introduction}

\input{sections/02_related_work_and_positioning}

\input{sections/03_marginal_preserving_gauges}

\input{sections/04_modified_equation_density_bias}

\input{sections/05_density_active_component}

\input{sections/06_residual_based_timestep_allocation}

\input{sections/07_gauge_consequences}

\input{sections/08_controlled_experiments}

\input{sections/10_conclusion}

\clearpage
\bibliographystyle{preprint_natbib}
\bibliography{references}

\appendix
\renewcommand{\thefigure}{\Alph{section}.\arabic{figure}}
\renewcommand{\thetable}{\Alph{section}.\arabic{table}}
\renewcommand{\thealgorithm}{\Alph{section}.\arabic{algorithm}}
\renewcommand{\theHfigure}{app.\Alph{section}.\arabic{figure}}
\renewcommand{\theHtable}{app.\Alph{section}.\arabic{table}}
\renewcommand{\theHalgorithm}{app.\Alph{section}.\arabic{algorithm}}
\setcounter{figure}{0}
\setcounter{table}{0}
\setcounter{algorithm}{0}

\section*{Appendix Contents}
\setlength{\columnsep}{1.8em}
\begin{multicols}{2}
\small
\raggedcolumns
\noindent\hyperref[app:edm-cifar-details]{\textbf{A.} Pretrained EDM CIFAR-10 Details}\par
\noindent\hyperref[app:synthetic-experiments]{\textbf{B.} Exact Mechanism Checks}\par
\noindent\hyperref[app:controlled-extra]{\textbf{C.} Hard-GMM Allocation Study}\par
\noindent\hyperref[app:unified-new-experiments]{\textbf{D.} Unified Baselines and Estimator Validation}\par
\noindent\hyperref[app:m1-experiments]{\textbf{E.} Additional Theorem-Linked Quantitative Checks}\par
\noindent\hyperref[app:applications]{\textbf{F.} Applications and Optimization Relevance}\par
\noindent\hyperref[app:local-ot]{\textbf{G.} Local Optimal-Transport Interpretation}\par
\noindent\hyperref[app:solver-aware-limits]{\textbf{H.} Solver-Aware Extensions and Variational Limits}\par
\noindent\hyperref[app:projected-gauge]{\textbf{I.} Projected Gauge Optimization}\par
\noindent\hyperref[app:gaussian-gauge]{\textbf{J.} Exact and Obstructed Gaussian Gauges}\par
\noindent\hyperref[app:proofs]{\textbf{K.} Additional Projection Details}\par
\noindent\hyperref[app:propagation-proof]{\textbf{L.} Propagation of the Density Defect and Cost Accounting}\par
\noindent\hyperref[app:mixtures]{\textbf{M.} Mixtures and Practical Parameterizations}\par
\noindent\hyperref[app:examples]{\textbf{N.} Additional Explicit Examples}\par
\noindent\hyperref[app:connection]{\textbf{O.} Connection Viewpoint}\par
\end{multicols}
\normalsize
\medskip

\input{appendices/Q_edm_cifar_details}

\input{appendices/D_analytic_synthetic_experiments}

\input{appendices/E_hard_gmm_experimental_details}

\input{appendices/M_unified_baselines_and_estimator_validation}

\input{appendices/P_theorem_linked_experiments}

\input{appendices/H_applications_and_optimization_relevance}

\FloatBarrier
\input{appendices/N_local_optimal_transport_interpretation}

\input{appendices/O_solver_aware_variational_limits}

\input{appendices/A_projected_gauge_optimization}

\input{appendices/B_exact_and_obstructed_gaussian_gauges}

\input{appendices/K_additional_projection_details}

\input{appendices/L_propagated_error_and_cost_accounting}

\input{appendices/C_mixtures_and_practical_parameterizations}

\input{appendices/I_additional_explicit_examples}

\input{appendices/J_connection_viewpoint}

\end{document}

%% file: math_commands.tex
\usepackage{amsmath,amsfonts,bm}

\def\eqref#1{equation~\ref{#1}}

\def\1{\bm{1}}

\DeclareMathAlphabet{\mathsfit}{\encodingdefault}{\sfdefault}{m}{sl}
\SetMathAlphabet{\mathsfit}{bold}{\encodingdefault}{\sfdefault}{bx}{n}

\newcommand{\E}{\mathbb{E}}

\newcommand{\R}{\mathbb{R}}



%% file: sections/00_abstract.tex
\begin{abstract}
Low-NFE generative ODE sampling is constrained not only by solver order, but
also by how a fixed budget of network evaluations is distributed along the
flow.  We develop a density-level backward-error framework that identifies the
component of a solver's modified-equation defect that changes the evolving
marginal, while excluding solenoidal trajectory distortion that is invisible in
Wasserstein tangent geometry.  The resulting residual induces a
solver-order-aware compute-allocation rule and, through gauge freedom, an
objective for selecting numerically favorable velocity representatives of the
same density path.  This separates our approach from schedule search driven by
trajectory or endpoint surrogates and from higher-order solver design.  Exact
Gaussian and mixture models establish the mechanism, including cases where
curvature-based allocation fails and continuum-optimal transport is not
discretization-optimal.  On a pretrained CIFAR-10 EDM probability-flow ODE,
density-aware allocation consistently improves fixed-update Euler sampling in
the low-NFE regime.  Matched-seed convergence to a dense Euler reference
confirms that the method accelerates approximation of the same high-NFE flow
rather than changing the asymptotic sampler.
\end{abstract}

%% file: sections/01_introduction.tex
\section{Introduction}

Diffusion and flow-based generators sample by numerically integrating a learned
time-dependent velocity field.  Let $v_t$ be a reference velocity transporting a
prescribed density path $p_t$ through
\begin{equation}
    \partial_t p_t+\nabla\cdot(p_tv_t)=0.
    \label{eq:ce}
\end{equation}
When only a few network evaluations are available, generation quality depends
not only on the learned field and solver formula, but also on where those
evaluations are placed along the path.

Timestep design has therefore become an optimization problem in its own right.
Align Your Steps optimizes schedules for a given solver, model, and dataset;
optimized-time-step methods fit nodes against a reference ODE solution; LD3
learns a discretization for a pretrained sampler; Optimal Stepsize Distillation
uses dynamic programming over reference trajectories; and differentiable solver
search jointly tunes timesteps and solver coefficients
\citep{Sabour2024AYS,Xue2024OptimizedTimesteps,Tong2025LD3,Pei2025OptimalStepsize,Wang2025DSS}.
These methods show that evaluation placement matters.  They do not, however,
isolate which part of a local numerical defect changes the evolving marginal,
nor do they use the fact that several velocity fields can generate the same
density path.

Three design variables should therefore be separated.  \emph{Solver design}
chooses the update formula.  \emph{Schedule design} chooses where that formula
evaluates the field.  \emph{Representative design} chooses which velocity field
to discretize when several fields generate the same continuum marginals.  Our
paper connects the last two through density-level backward error.

To make this non-uniqueness explicit, define
$\mathcal U_p=\{u:\nabla\cdot(p_tu_t)=0\ \text{for almost every }t\}$.  If
$u\in\mathcal U_p$, then $w=v+u$ generates the same continuum density curve as
$v$, although a finite numerical map need not preserve this equivalence.  A
trajectory-level defect can therefore contain substantial particle motion that
is invisible to the marginal.  The relevant question is not simply
\emph{where is the ODE difficult?}, but \emph{which part of the solver defect
moves probability mass?}

This also separates continuum and finite-step optimality.  Probability-flow
endpoint maps coincide with quadratic Monge maps in important Gaussian and
symmetric settings \citep{Khrulkov2023DDPMOT}, but not in general
\citep{LavenantSantambrogio2022}.  Here the density curve $p_t$ is fixed.  We
ask which continuum-equivalent velocity representative and which timestep grid
make the \emph{discrete} transport maps most accurate in $W_2$.  Even for
isotropic Gaussians, the minimum-action gradient representative can be less
accurate for Euler than a non-gradient marginal-preserving representative.

Let $\Psi_h^w$ be one numerical step for a representative $w$, and suppose an
order-$r$ method admits the local modified field
\begin{equation}
    w_h=w+h^r\mathcal E_r[w]+O(h^{r+1}).
    \label{eq:intro-modified-general}
\end{equation}
Write $P_{\nabla,p_t}$ for the $L^2(p_t)$-orthogonal projection onto gradient
fields.  This projection keeps the Wasserstein tangent component that changes
the marginal and removes marginal-invisible solenoidal motion.  Our central
local result is
\begin{align}
(\Psi_h^w)_\#p_t-p_{t+h}
&=-h^{r+1}\nabla\cdot\!\left(p_t\mathcal E_r[w]\right)
  +o(h^{r+1}), \label{eq:intro-density-general}\\
W_2\!\left((\Psi_h^w)_\#p_t,p_{t+h}\right)
&=h^{r+1}
\left\|P_{\nabla,p_t}\mathcal E_r[w]\right\|_{L^2(p_t)}
+o(h^{r+1}). \label{eq:intro-w2-general}
\end{align}
For Euler, $\mathcal E_1[w]=-\tfrac12a_w$ with
$a_w=\partial_tw+\nabla w\,w$.  Thus the projected modified-equation defect,
rather than the full trajectory defect, determines the leading marginal error.

This separates continuum transport optimality from two finite-solver questions:
where to place evaluations and which continuum-equivalent field to discretize.

\paragraph{Contributions.}
\begin{enumerate}
\item \textbf{Density-level backward error.}
We project the first nonzero modified-equation differential onto the weighted
Wasserstein tangent space.  Equations~(\ref{eq:intro-density-general})--
(\ref{eq:intro-w2-general}) identify its norm with the leading local marginal
$W_2$ defect and discard solenoidal motion that is invisible to the density.
\item \textbf{Finite-budget compute allocation.}
Let $R_r(t)=\|P_{\nabla,p_t}\mathcal E_r[w]\|_{L^2(p_t)}$, and let $K(1,t)$
denote the amplification of a density perturbation introduced at time $t$ and
propagated to terminal time.  Then
\[
 n^\star(t)\propto[K(1,t)R_r(t)]^{1/(r+1)}.
\]
The solver order fixes the exponent; the associated quantile grid exactly
minimizes the finite-$N$ local surrogate defined in Section~\ref{sec:schedule-theory},
and a R\'enyi identity quantifies residual misspecification.
\item \textbf{Solver-aware representative selection.}
Among fields $w=v+u$ with $u\in\mathcal U_p$, we define a continuum solver
objective.  For Euler, renormalized discrete local-$W_2$ objectives converge
uniformly---and hence $\Gamma$-converge---to this functional.  Gaussian paths
give explicit controls and separate continuum minimum action from finite-step
solver accuracy.
\item \textbf{Mechanism-aligned validation.}
Analytic Gaussian and mixture experiments test the projected defect,
propagation, and representative-selection predictions.  On pretrained CIFAR-10
EDM, we keep the Euler update fixed and change only the timestep grid.  A
dense-reference study verifies faster convergence to the same high-NFE Euler
limit rather than a change of the target sampler.
\end{enumerate}

\begin{keybox}
\textbf{Two optimization problems.}
Representative selection minimizes
\[
\min_{u\in\mathcal U_p}\mathcal B(v+u),
\qquad
\mathcal B(w)=\frac12\int_0^1
\|P_{\nabla,p_t}a[w]\|_{L^2(p_t)}^2\,dt .
\]
For an order-$r$ solver, let $W_r(t)=K(1,t)R_r(t)$.  Timestep allocation solves
\[
\min_{n>0}\int_0^1W_r(t)n(t)^{-r}\,dt
\quad\text{s.t.}\quad \int_0^1n(t)\,dt=N.
\]
The first problem chooses a continuum-equivalent velocity field for a fixed
solver; the second distributes that solver's evaluation budget along the path.
\end{keybox}

The empirical comparisons keep this separation explicit.  Euler schedules use
the same network, update rule, NFE, and sampling cost.  DPM++ and the original
EDM Heun sampler are shown only as stronger-solver context: they change the
integration formula, whereas our method changes where a fixed first-order
solver spends its evaluations.  The offline allocation profile is reusable
across generated samples and NFE choices.

%% file: sections/02_related_work_and_positioning.tex
\section{Related Work and Positioning}

\paragraph{Generative flows and optimal transport.}
Continuous normalizing flows and neural ODEs formulate generation as numerical
integration of a learned velocity field \citep{Chen2018NeuralODE}.  Diffusion
and score-based models admit probability-flow ODE descriptions
\citep{Ho2020DDPM,Song2021SDE}, while flow matching directly regresses vector
fields for prescribed probability paths
\citep{Lipman2023FlowMatching,Albergo2023StochasticInterpolants}.  Rectified and
multisample flow matching simplify transport through straighter trajectories or
improved couplings \citep{Liu2023RectifiedFlow,Pooladian2023MultisampleFM}.
Probability-flow endpoint maps agree with quadratic optimal transport in
important structured cases \citep{Khrulkov2023DDPMOT}, but not universally
\citep{LavenantSantambrogio2022}.  These works concern the continuum path or
endpoint map.  We fix the marginal curve and optimize the numerical realization
of its gauge-equivalent lifts.  This also differs from Wasserstein spline and
second-order transport models, which optimize the curve itself
\citep{BenamouGallouetVialard2019}.

\paragraph{Solver and schedule optimization.}
Backward error analysis and B-series modified equations are classical in
geometric numerical integration \citep{Hairer2006GNI} and have been applied to
diffusion sampling \citep{Gao2023BEA}; the probability-flow ODE also admits
nonasymptotic convergence analyses \citep{Chen2023PFODE}.  Solver design ranges
from the EDM Heun construction to DPM-Solver and empirical-statistics
refinements \citep{Karras2022EDM,Lu2022DPMSolver,Zheng2023DPMSolverV3}.

A complementary line optimizes the discretization itself.  Align Your Steps
uses stochastic-calculus objectives specialized to a solver, model, and data
distribution \citep{Sabour2024AYS}.  \citet{Xue2024OptimizedTimesteps} optimize
time nodes by comparing the numerical and reference ODE solutions.  LD3 learns
a lightweight discretization for pretrained samplers \citep{Tong2025LD3}, while
Optimal Stepsize Distillation applies dynamic programming to reference
trajectories \citep{Pei2025OptimalStepsize}.  Differentiable Solver Search
jointly searches timesteps and solver coefficients \citep{Wang2025DSS}.
Adaptive step control, Wasserstein-bounded schedules, continuous-time clock
control, and instance-aware grids further enlarge the schedule class
\citep{Chen2024Adaptive,Jo2026SamplingDesign,HuangTangZhou2026ART,YuanEtAl2026InstanceAware}.
Our contribution is orthogonal to how that class is searched: we derive the
local quantity to allocate by projecting the first modified-equation
differential through the continuity equation.  This yields a marginal residual,
a solver-order exponent, and an additional optimization variable---the
marginal-preserving velocity representative.

\paragraph{Gauge freedom and marginal observability.}
The non-uniqueness of a velocity field realizing a fixed density path is
increasingly visible in diffusion and flow modeling.  \citet{HorvatPfister2024}
analyze gauge freedom and conservativity; recent work separates gradient and
solenoidal components by marginal observability \citep{Khelifa2026OnlyGradients}.
EDDY uses divergence-free freedom for marginal-preserving particle diversity
\citep{Vinograd2026EDDY}, while Gauge Flow Models, Flux Matching, and
Wasserstein-flow reinterpretations treat the field itself as a design variable
\citep{StrunkAssam2025GaugeFlow,PaoHuang2026FluxMatching,Vuong2025WGF}.
Curvature minimization and strain/vorticity decompositions instead diagnose
trajectory-level integration difficulty \citep{Lee2023Curvature,TaoChoi2026Strain}.
We apply the weighted Hodge projection to the \emph{solver's modified-equation
defect}: only its gradient component contributes to the leading marginal
$W_2$ error, and the remaining gauge freedom can be optimized for numerical
integrability.

\paragraph{Optimization viewpoint.}
The schedule layer is a continuous resource-allocation problem, whereas
representative selection is a PDE-constrained optimization problem with
weighted Poisson projections.  This connects the construction to computational
optimal transport and Wasserstein geometry
\citep{BenamouBrenier2000,Otto2001,AmbrosioGigliSavare2008,Villani2009,Santambrogio2015}
and to PDE-constrained optimization \citep{Hinze2009PDEOpt}.  The local theorem
provides an exact solver-level bridge, while the variational limit connects the
discrete and continuum optimization problems.  Exact Gaussian analyses of
diffusion discretization provide a complementary controlled perspective
\citep{PierretGalerne2025Gaussian}.

%% file: sections/03_marginal_preserving_gauges.tex
\section{Marginal-Preserving Gauges}

\paragraph{Analytic setting.}
Throughout the main text we work on either the flat torus $\mathbb T^d$, a smooth compact boundaryless domain, or $\R^d$ with sufficient decay so that all integrations by parts below have no boundary contribution.  We assume $p_t>0$, $\int p_t=1$, and enough smoothness of $p_t$ and the vector fields to justify the displayed identities; a convenient sufficient condition is $p\in C^2_{t,x}$ and $v,u\in C^1_tC^2_x$ with bounded derivatives on compact time intervals.  On $\R^d$ test functions are compactly supported unless stronger decay is imposed.  All divergence constraints may equivalently be read in the weak sense.  Let $\mathcal T=C^\infty(M)/\mathbb R$ on a compact boundaryless domain and $\mathcal T=C_c^\infty(\R^d)$ in the Euclidean case.

\begin{definition}[Solenoidal gauge class]
For a fixed density $p_t$, define
\begin{equation}
    \cU_{p_t}
    =\left\{u\in L^2(p_t):\int p_t\langle u,\nabla\phi\rangle\,dx=0
    \quad \forall \phi\in\mathcal T\right\} .
    \label{eq:weak-solenoidal}
\end{equation}
For smooth fields this is equivalent to $\nabla\cdot(p_tu_t)=0$ in distributions.  Over time we write $\cU_p=\{u_t\in\cU_{p_t}\text{ for a.e. }t\}$ and the representative fiber as $\cG(v,p)=v+\cU_p$.
\end{definition}

\begin{proposition}[Gauge invariance]
If $v_t$ satisfies \eqref{eq:ce} and $u_t\in\cU_p$, then $w_t=v_t+u_t$ satisfies
\begin{equation}
    \partial_t p_t+\nabla\cdot(p_tw_t)=0 .
\end{equation}
\end{proposition}

\begin{proof}
By linearity,
\[
\partial_t p_t+\nabla\cdot(p_t(v_t+u_t))
=
\partial_t p_t+\nabla\cdot(p_tv_t)+\nabla\cdot(p_tu_t)=0.
\]
\end{proof}

The proposition is elementary, but it defines the feasible set of the optimization problem.  In one dimension, this feasible set is trivial under standard decay assumptions: $\partial_x(p_tu_t)=0$ implies $p_tu_t=c(t)$, and decay gives $c(t)=0$.  Hence nontrivial gauge optimization is intrinsically multidimensional.

%% file: sections/04_modified_equation_density_bias.tex
\section{Solver-Aware Density Defect}
\label{sec:solver-density-defect}

Let $\Psi_h^w$ denote one step of a numerical method applied to the
non-autonomous ODE $\dot x_t=w_t(x_t)$.  We use the extended system
$(x,\tau)'=(w_\tau(x),1)$ when invoking modified-equation expansions.

\begin{definition}[Leading modified-equation differential]
An order-$r$ method is said to admit a first modified differential
$\mathcal E_r[w]$ on a class $\mathcal A$ if its local modified field satisfies
\begin{equation}
    w_h=w+h^r\mathcal E_r[w]+h^{r+1}\rho_h[w],
    \qquad
    \sup_{w\in\mathcal A,\,h<h_0}\|\rho_h[w]\|_{C_x^1}<\infty .
    \label{eq:modified-field-general}
\end{equation}
\end{definition}

\begin{theorem}[Solver-aware density-level backward error]
\label{thm:density-bea}
Suppose $w_t$ transports $p_t$ and \eqref{eq:modified-field-general} holds.
Then
\begin{equation}
\partial_t p_t+\nabla\cdot(p_tw_h)
=h^r\nabla\cdot\!\left(p_t\mathcal E_r[w]\right)+O(h^{r+1})
\label{eq:residual-general}
\end{equation}
in distributions and in any negative Sobolev norm controlled by the stated
remainder.  Over one step,
\begin{equation}
(\Psi_h^w)_\#p_t-p_{t+h}
=-h^{r+1}\nabla\cdot\!\left(p_t\mathcal E_r[w]\right)
+o(h^{r+1}).
\label{eq:one-step-density-general}
\end{equation}
\end{theorem}

The first identity is immediate after subtracting the continuity equation;
the second integrates the $h^r$ velocity perturbation for time $h$.
Appendix~\ref{app:solver-aware-limits} gives a weak proof and records the
uniformity needed later.

\paragraph{Euler specialization.}
The exact solution has
\[
x(t+h)=x+hw_t(x)+\frac{h^2}{2}a_w(x,t)+O(h^3),
\qquad
a_w=\partial_tw+\nabla w\,w .
\]
Hence explicit Euler has
\begin{equation}
\mathcal E_1[w]=-\frac12a_w,
\qquad
w_h=w-\frac h2a_w+O(h^2),
\label{eq:modified-field}
\end{equation}
and \eqref{eq:residual-general} becomes
\begin{equation}
\partial_t p_t+\nabla\cdot(p_tw_h)
=-\frac h2\nabla\cdot(p_ta_w)+O(h^2).
\label{eq:residual}
\end{equation}
Thus the familiar Euler acceleration is one instance of a general object:
the first modified-equation differential seen through weighted divergence.
For midpoint or Heun, the first nonzero object is
$\mathcal E_2[w]$ and enters at order $h^2$ in the modified field.

%% file: sections/05_density_active_component.tex
\section{Density-Active Wasserstein Geometry}
\label{sec:density-active-geometry}

Equip $L^2(p_t)$ with
$\langle r,s\rangle_{p_t}=\int p_t\langle r,s\rangle\,dx$ and define
\[
\mathcal G_{p_t}
=\overline{\{\nabla\phi:\phi\in\mathcal T\}}^{L^2(p_t)},
\]
where $\mathcal T$ is the test-function space fixed in Section~3.  The gauge
definition gives the orthogonal splitting
\begin{equation}
L^2(p_t)=\mathcal G_{p_t}\oplus\mathcal U_{p_t}.
\label{eq:helmholtz-closed}
\end{equation}
Write $P_{\nabla,p_t}$ for the gradient projection.  Weakly,
$P_{\nabla,p_t}z=\nabla\phi$ is characterized by
\begin{equation}
\int p_t\langle\nabla\phi,\nabla\psi\rangle\,dx
=\int p_t\langle z,\nabla\psi\rangle\,dx ,
\label{eq:poisson-weak}
\end{equation}
or, in strong form,
\begin{equation}
\nabla\cdot(p_t\nabla\phi)=\nabla\cdot(p_tz).
\label{eq:poisson-grad}
\end{equation}

\begin{definition}[Weighted negative norm]
For a mean-zero distribution $f$, let
\begin{equation}
\|f\|_{\dot H^{-1}(p_t)}
=\sup_{\|\nabla\psi\|_{L^2(p_t)}\le1}\langle f,\psi\rangle .
\label{eq:hminus-def}
\end{equation}
\end{definition}

\begin{proposition}[Density-active projection]
\label{prop:hminus-main}
For every sufficiently regular field $z$,
\begin{equation}
\nabla\cdot(p_tz)
=\nabla\cdot\!\left(p_tP_{\nabla,p_t}z\right),
\qquad
\|\nabla\cdot(p_tz)\|_{\dot H^{-1}(p_t)}
=\|P_{\nabla,p_t}z\|_{L^2(p_t)} .
\label{eq:hminus}
\end{equation}
\end{proposition}
\begin{proof}
The solenoidal component is orthogonal to every test gradient, so it vanishes
both under weighted divergence and in the dual norm.  The remaining supremum
is the Hilbert norm of the gradient projection.
\end{proof}

Apply the proposition to $z=\mathcal E_r[w]$ and define the solver-aware local
residual
\begin{equation}
R_r[w](t)
:=\|P_{\nabla,p_t}\mathcal E_r[w]\|_{L^2(p_t)}
=\|\nabla\cdot(p_t\mathcal E_r[w])\|_{\dot H^{-1}(p_t)}.
\label{eq:solver-residual}
\end{equation}

\begin{theorem}[Local Wasserstein defect for an order-$r$ method]
\label{thm:local-w2-defect}
Under the regular Wasserstein-chart and uniformity assumptions of
Appendix~\ref{app:solver-aware-limits},
\begin{equation}
W_2\!\left((\Psi_h^w)_\#p_t,p_{t+h}\right)
=h^{r+1}R_r[w](t)+o(h^{r+1}).
\label{eq:local-w2-defect}
\end{equation}
For Euler this is
\begin{equation}
W_2\!\left((\operatorname{Id}+hw_t)_\#p_t,p_{t+h}\right)
=\frac{h^2}{2}
\|P_{\nabla,p_t}a_w\|_{L^2(p_t)}+o(h^2).
\label{eq:local-w2-euler}
\end{equation}
\end{theorem}

The theorem turns the modified-equation residual into a Wasserstein tangent
norm rather than a heuristic indicator.  It also defines a solver-dependent
gauge objective
\begin{equation}
\mathcal B_r(w)=\frac12\int_0^1R_r[w](t)^2\,dt.
\label{eq:generic-bias-objective}
\end{equation}
For Euler, $\mathcal B_1=\mathcal B/4$, where
\begin{equation}
\mathcal B(w)
=\frac12\int_0^1
\|P_{\nabla,p_t}a_w\|_{L^2(p_t)}^2\,dt;
\label{eq:bias-objective}
\end{equation}
the constant does not affect representative selection.

\begin{keybox}
\textbf{Solver-aware transport chain.}
\[
\mathcal E_r[w]
\xrightarrow{\;P_{\nabla,p_t}\;} R_r(t)
\xrightarrow{\;K(1,t)\;} K(1,t)R_r(t)
\xrightarrow{\;\int n=N\;} n^\star(t).
\]
Only the horizontal component contributes to the leading marginal defect.
Large $p_t$-solenoidal modified differentials may bend particle trajectories
without producing leading density error.
\end{keybox}

%% file: sections/06_residual_based_timestep_allocation.tex
\section{Solver-Order Timestep Allocation}
\label{sec:schedule-theory}

Let $R_r(t)$ be the geometric residual in \eqref{eq:solver-residual}.  A local
defect may be amplified before terminal time.  For the transport evolution
$U_w(1,t)$, assume
\begin{equation}
\|U_w(1,t)f\|_X\le K(1,t)\|f\|_Y .
\label{eq:transport-stability}
\end{equation}

\begin{theorem}[Propagated order-$r$ density defect]
\label{thm:propagated-defect}
Under the differentiability and modified-equation assumptions in
Appendices~\ref{app:propagation-proof} and~\ref{app:solver-aware-limits}, a
smooth local step profile $h(t)$ satisfies
\begin{equation}
\|p_1^{\rm num}-p_1\|_X
\le
\int_0^1K(1,t)h(t)^r
\|\nabla\cdot(p_t\mathcal E_r[w])\|_Y\,dt
+O(\|h\|_\infty^{r+1}).
\label{eq:order-r-final-bound}
\end{equation}
\end{theorem}

For the canonical source norm $Y=\dot H^{-1}(p_t)$, Proposition~\ref{prop:hminus-main} identifies the integrand with $R_r(t)$.  The terminally weighted residual is therefore
$W_r(t)=K(1,t)R_r(t)$.  If $n(t)\,dt$ steps are placed in $dt$, then
$h(t)\simeq n(t)^{-1}$ and the accumulated leading defect is
\begin{equation}
\mathcal J_r[n]=\int_0^1W_r(t)n(t)^{-r}\,dt,
\qquad
\int_0^1n(t)\,dt=N.
\label{eq:schedule-opt}
\end{equation}
For a generic exponent $q>0$ we use
$\mathcal J_q^W[n]=\int_0^1W(t)n(t)^{-q}\,dt$.

\begin{theorem}[Solver-order allocation and finite-$N$ surrogate]
\label{thm:schedule-opt}
Let $W>0$, $q>0$, and $\alpha=(q+1)^{-1}$.  The unique continuous minimizer
of $\int Wn^{-q}$ under $\int n=N$ is
\begin{equation}
n^\star(t)=
\frac{NW(t)^\alpha}{Z},
\qquad
Z=\int_0^1W(s)^\alpha\,ds .
\label{eq:nstar}
\end{equation}
For any partition $\pi=\{0=t_0<\cdots<t_N=1\}$ define
\begin{equation}
\mathcal S_q(\pi)
=\sum_{k=0}^{N-1}
\left(\int_{t_k}^{t_{k+1}}W(t)^\alpha\,dt\right)^{q+1}.
\label{eq:discrete-surrogate}
\end{equation}
Then
\begin{equation}
\mathcal S_q(\pi)\ge\frac{Z^{q+1}}{N^q},
\label{eq:discrete-quantile-bound}
\end{equation}
with equality exactly when
\begin{equation}
\int_0^{t_k}W(t)^\alpha\,dt=\frac{k}{N}Z .
\label{eq:exact-quantile-grid}
\end{equation}
\end{theorem}

The continuous law follows from the KKT equation
$-qWn^{-q-1}+\lambda=0$.  For the finite-$N$ surrogate $\mathcal S_q$, strict convexity applied to the
interval masses $a_k=\int_{t_k}^{t_{k+1}}W^\alpha$ gives the exact quantile
minimizer.  This is exact finite-$N$ \emph{surrogate} optimality; it is not a
claim that the same grid exactly minimizes the full nonlinear terminal error.
For an order-$r$ solver, $q=r$, so
\begin{equation}
n^\star(t)\propto[K(1,t)R_r(t)]^{1/(r+1)}.
\label{eq:order-r-allocation}
\end{equation}
Euler, second-order, and third-order methods therefore use powers
$1/2$, $1/3$, and $1/4$, respectively.

The quantile rule in \eqref{eq:exact-quantile-grid} directly converts the
continuous optimum into a realizable mesh; its local step asymptotics and
$O(N^{-2})$ endpoint approximation are stated in
Appendix~\ref{app:solver-aware-limits}.

\subsection{Allocation robustness and propagation}
Suppose a schedule is built from $\widehat W$ instead of $W$.  Let
\begin{equation}
\rho=\frac{W^\alpha}{Z},
\qquad
\widehat\rho=\frac{\widehat W^\alpha}{\widehat Z}.
\label{eq:schedule-densities}
\end{equation}

\begin{theorem}[Exact schedule degradation]
\label{thm:schedule-robustness}
For $\widehat n=N\widehat\rho$,
\begin{equation}
\frac{\mathcal J_q^W[\widehat n]}{\mathcal J_q^W[n^\star]}
=\int_0^1\rho^{q+1}\widehat\rho^{-q}\,dt
=\exp\!\left(qD_{q+1}(\rho\Vert\widehat\rho)\right).
\label{eq:renyi-robustness}
\end{equation}
For Euler,
\begin{equation}
\frac{\mathcal J_1^W[\widehat n]}{\mathcal J_1^W[n^\star]}
=1+\chi^2(\rho\Vert\widehat\rho).
\label{eq:chi2-robustness}
\end{equation}
If $|\log(\widehat W/W)|\le\varepsilon$, the ratio is at most
$\exp(2q\varepsilon/(q+1))$.
\end{theorem}

Thus schedule quality depends on the normalized allocation density, not on
the absolute scale of the residual.  The regularization
$W_\varepsilon=\varepsilon+R$ ensures positivity and controls sensitivity near
zeros.

Propagation weighting is structurally consequential: even for a constant local
residual, Appendix~\ref{app:solver-aware-limits} constructs smooth scalar
flows for which the residual-only objective ratio grows as
$\sqrt{M}/4\to\infty$ relative to the propagation-aware optimum.

The proofs of Theorems~\ref{thm:schedule-opt} and
\ref{thm:schedule-robustness}, together with the mesh asymptotics and the
propagation-separation construction, are collected in
Appendix~\ref{app:solver-aware-limits}.

\paragraph{Practical residual estimation.}
The theory uses $R_r(t)=\|P_{\nabla,p_t}\mathcal E_r[w]\|_{L^2(p_t)}$.
Controlled Euler experiments also use the score--divergence proxy
\begin{equation}
\widehat R_{\rm dens}(t)=\|p_t^{-1}\nabla\cdot(p_ta_w)\|_{L^2(p_t)},
\label{eq:score-div-proxy}
\end{equation}
which upper-controls, but is not identical to, the geometric norm under a
weighted Poincar\'e inequality.  CIFAR-10 instead uses a Galerkin projection
and empirical propagation; Appendix~\ref{app:unified-new-experiments} validates
a particle-only estimator against the exact controlled residual.

%% file: sections/07_gauge_consequences.tex
\section{Gauge Optimization and Its Discrete Limit}
\label{sec:gauge-consequences-main}

For Euler, the representative-selection problem is
\begin{equation}
\min_{u\in\mathcal U_p}\mathcal B(v+u),
\qquad
\mathcal B(w)=\frac12\int_0^1
\|P_{\nabla,p_t}a[w]\|_{L^2(p_t)}^2\,dt .
\label{eq:main-gauge-problem}
\end{equation}
The first variation, projected descent, KKT system, and weighted-Poisson
projection are given in Appendix~\ref{app:projected-gauge}.  The feasible
direction is the solenoidal projection of the ambient gradient, so the
optimization remains inside the fiber of fields that realize the same density
path.

Why optimize a continuum functional when the target is a finite-step solver?
For a partition $\pi$, define the discrete local-$W_2$ objective
\begin{equation}
\mathcal B_\pi(w)
=\sum_{k=0}^{N-1}\frac{2}{h_k^3}
W_2^2\!\left(
(\operatorname{Id}+h_kw_{t_k})_\#p_{t_k},p_{t_{k+1}}
\right),
\qquad h_k=t_{k+1}-t_k .
\label{eq:discrete-gauge-objective}
\end{equation}

\begin{theorem}[$\Gamma$-limit of discrete Euler gauge optimization]
\label{thm:gauge-gamma}
Let $\mathcal A\subset v+\mathcal U_p$ be compact in a topology $\tau$.
Assume the Euler local-$W_2$ expansion is uniform on $\mathcal A$, its
$O(h^3)$ remainder is uniformly bounded, and
$t\mapsto\|P_{\nabla,p_t}a[w]\|^2$ is uniformly Lipschitz over
$w\in\mathcal A$.  Then
\begin{equation}
\sup_{w\in\mathcal A}
|\mathcal B_\pi(w)-\mathcal B(w)|
\le C|\pi|\longrightarrow0.
\label{eq:uniform-gauge-convergence}
\end{equation}
Consequently $\mathcal B_\pi\xrightarrow{\Gamma}\mathcal B$ in $\tau$.
Every convergent sequence of discrete minimizers has a continuum-minimizing
limit.  The same conclusion holds on noncompact admissible sets when an added
regularizer makes the objectives equicoercive.
\end{theorem}

Thus the continuum representative-selection problem is the variational limit
of the discrete solver objective, not merely an analogy suggested by backward
error.  Appendix~\ref{app:solver-aware-limits} gives the proof and the
regularized equicoercive form.

Classical optimal transport selects a horizontal lift by minimizing continuum
kinetic energy.  Here the prescribed marginal curve is held fixed, while the
chosen numerical method induces a second horizontal criterion: minimize the
local transport defect of its discrete maps.  This complements endpoint-map
interpretations of probability-flow ODEs in Gaussian settings
\citep{Khrulkov2023DDPMOT} and the classical Otto--Benamou--Brenier geometry
\citep{BenamouBrenier2000,Otto2001}.

\paragraph{Gaussian exact control.}
For $p_t=\mathcal N(0,\Sigma_t)$ and $w_t(x)=B_tx$, continuum transport and
zero leading Euler bias reduce to
\begin{equation}
B_t\Sigma_t+\Sigma_tB_t^\top=\dot\Sigma_t,
\qquad
2B_t\Sigma_tB_t^\top=\ddot\Sigma_t .
\label{eq:main-gaussian-summary}
\end{equation}
For $\Sigma_t=s(t)^2I_d$, feasible lifts are
$B_t=(\dot s/s)I_d+A_t$, $A_t^\top=-A_t$, and cancellation requires
\begin{equation}
A_t^2=-\frac{\ddot s(t)}{s(t)}I_d.
\label{eq:main-parity-summary}
\end{equation}
When $\ddot s/s>0$, exact skew gauges exist in even dimension, whereas one
scalar direction remains unavoidable in odd dimension; the full
classification is in Appendix~\ref{app:gaussian-gauge}.

\begin{takeawaybox}
\textbf{Continuum OT optimality is not finite-step solver optimality.}
The radial Gaussian lift $v_t=(\dot s/s)x$ is the gradient,
minimum-kinetic-energy representative, yet
\[
W_2((\operatorname{Id}+hv_t)_\#p_t,p_{t+h})
=\frac{\sqrt d}{2}|\ddot s(t)|h^2+O(h^3).
\]
In even dimension the admissible non-gradient lift
$[(\dot s/s)I+\sqrt{\ddot s/s}\,J]x$, with $J^\top=-J$ and $J^2=-I$,
cancels the leading covariance defect and achieves $O(h^3)$.  The same
continuum density path therefore admits a Brenier-favored lift and a distinct
Euler-favored lift.  The gauge fiber is a genuine finite-solver control space.
\end{takeawaybox}

%% file: sections/08_controlled_experiments.tex
\section{Controlled and Pretrained Validation}
\label{sec:main-experiments}

\paragraph{Controlled checks.}
Representative selection is tested in exact and controlled models; the
pretrained experiment below isolates timestep allocation.  Exact Gaussian
paths recover the predicted local-$W_2$ orders and the even/odd gauge obstruction.  Analytic mixture paths verify the finite-$N$ quantile
surrogate, propagation weighting, and robustness of the square-root allocation.
A unified hard-GMM protocol compares density-active scheduling with embedded
Euler--Heun refinement, Jo--Choi velocity-variation resampling, a particle-only
dual residual, propagation-aware weighting, a trajectory-coupling oracle, and
fair-NFE Heun.  These comparisons support competitiveness rather than uniform
dominance; the same appendix validates the particle-only residual against the
exact controlled profile.  Full protocols and quantitative checks are in
Appendices~\ref{app:synthetic-experiments}, \ref{app:controlled-extra},
\ref{app:unified-new-experiments}, and~\ref{app:m1-experiments}.

\paragraph{Pretrained EDM CIFAR-10.}
We integrate the deterministic probability-flow ODE of the pretrained
CIFAR-10 EDM network \citep{Karras2022EDM} in the noise variable $\sigma$,
\[
  \frac{\dd x}{\dd\sigma}=f_\sigma(x)
  =\frac{x-D_\theta(x,\sigma)}{\sigma}.
\]
A single offline profile estimates the Euler modified-equation residual
$\mathcal E_1=(\partial_\sigma f_\sigma+Jf_\sigma f_\sigma)/2$, projects it onto
a random-feature Wasserstein-horizontal space, and weights it by empirical
terminal amplification.  Euler timesteps are quantiles of
$[K_{\rm emp}(\sigma)R_{\rm gal}(\sigma)]^{1/2}$ in the actual solver variable.

\begin{table}[t]
\caption{\textbf{Pretrained CIFAR-10 EDM: FID-50k / generation time.}
The first block keeps the Euler update fixed and changes only timestep placement.
The second block changes the solver class.  FID evaluation time is excluded.}
\label{tab:edm-cifar-equal-nfe-runtime}
\centering
\begingroup
\footnotesize
\setlength{\tabcolsep}{2.8pt}
\renewcommand{\arraystretch}{1.02}
\begin{adjustbox}{max width=\linewidth,center}
\begin{tabular}{lcccc}
\toprule
Method & 8 NFE & 16 NFE & 20 NFE & 35 NFE \\
\midrule
\multicolumn{5}{l}{\emph{Fixed Euler update: schedule comparison}}\\[-1mm]
Euler/Ours & \textbf{16.4420} / 441.8s & \textbf{6.9354} / 867.7s & \textbf{5.5274} / 1084.0s & 3.5285 / 1884.6s \\
Euler/log-$\sigma$ & 21.6984 / 442.2s & 7.6153 / 867.4s & 5.8316 / 1084.3s & \textbf{3.5154} / 1884.7s \\
Euler/EDM-$\rho$ & 22.2181 / 441.8s & 8.1022 / 867.2s & 6.2264 / 1082.5s & 3.7318 / 1885.4s \\
\midrule
\multicolumn{5}{l}{\emph{Stronger solver context}}\\[-1mm]
DPM++ & 8.0429 / 456.8s & 2.6684 / 895.6s & 2.3117 / 1114.4s & 2.0205 / 1939.6s \\
Original EDM Heun & -- & -- & -- & \textbf{1.8486} / 1901.2s \\
\bottomrule
\end{tabular}
\end{adjustbox}
\endgroup
\end{table}

Within Euler, the proposed grid reduces FID relative to log-$\sigma$ at 8, 16,
and 20 NFE and improves over EDM-$\rho$ at every displayed budget.  The gain
nearly vanishes by 35 NFE, where the first-order mesh is already dense.  These
rows are the direct test of the paper: network, update rule, NFE, seeds, and
runtime are matched, so the only design variable is timestep placement.

DPM++ and EDM Heun address a different axis: they improve the update formula
through multistep or second-order information.  Their lower FID therefore
measures solver design, not Euler-grid quality.  Our method instead optimizes
where a fixed stateless first-order solver spends its evaluations and produces
an offline profile reusable across seeds and budgets.  A matched-seed dense-grid
study in Appendix~\ref{app:edm-cifar-details} finds our grid closest to a
$1000$-step Euler reference at every tested budget from $8$ to $512$ NFE, while
all three Euler schedules become indistinguishable at $1000$ NFE.  Thus the gain
is faster convergence to the same Euler limit, complementary to higher-order
solver engineering.

\begin{figure}[t]
\centering
\begingroup
\small
\setlength{\tabcolsep}{5pt}
\renewcommand{\arraystretch}{1.08}
\begin{tabular}{@{}p{0.21\linewidth}p{0.75\linewidth}@{}}
\shortstack[l]{\textbf{Euler/EDM-$\rho$}\\[-1mm]{\footnotesize FID-50k 22.22}} & \includegraphics[width=\linewidth]{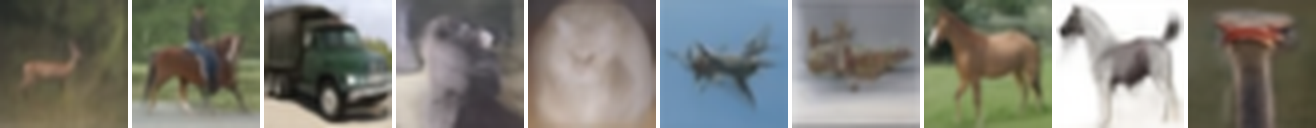}\\
\shortstack[l]{\textbf{Euler/log-$\sigma$}\\[-1mm]{\footnotesize FID-50k 21.70}} & \includegraphics[width=\linewidth]{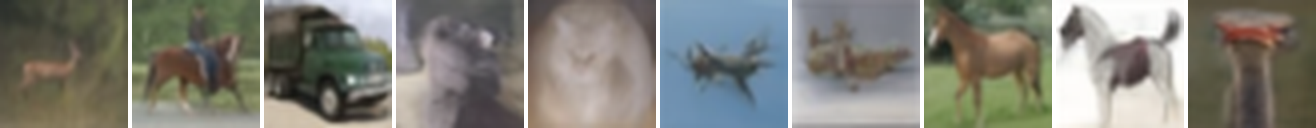}\\
\shortstack[l]{\textbf{Euler/Ours}\\[-1mm]{\footnotesize FID-50k 16.44}} & \includegraphics[width=\linewidth]{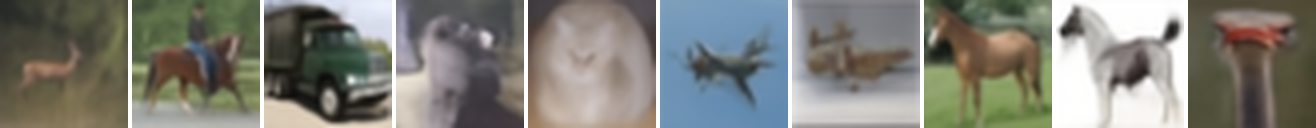}
\end{tabular}
\endgroup
\caption{Matched-seed CIFAR-10 samples at 8 NFE for the three Euler grids.
Every row uses the same network and update rule; only timestep placement changes.
The visual comparison therefore matches the fixed-solver claim tested by the
first block of Table~\ref{tab:edm-cifar-equal-nfe-runtime}.}
\label{fig:main-edm-cifar-samples}
\end{figure}

Appendix~\ref{app:edm-cifar-details} reports the profile construction,
dense-limit convergence, FID--time tradeoff, and the matched Euler comparison
at 16 NFE before the controlled and proof supplements.

%% file: sections/10_conclusion.tex
\section{Conclusion}
The density-active component of a solver's modified equation determines its
leading marginal transport error.  This yields a solver-order allocation law
and a representative-selection objective over marginal-preserving gauges, with
a discrete-to-continuum justification for Euler.  Controlled and pretrained
experiments validate the mechanism, while matched-seed dense-grid convergence
shows that the gain is faster approximation of the same high-NFE Euler limit.
DPM++ and EDM Heun delimit the scope: they change the integration formula;
ours allocates the budget of a fixed low-order solver.

%% file: appendices/Q_edm_cifar_details.tex
\section{Pretrained EDM CIFAR-10 Details}
\label{app:edm-cifar-details}

This appendix expands the image-model validation from
Section~\ref{sec:main-experiments}.  All methods use the same class-conditional
pretrained EDM network, seeds $0$--$49999$, batch size $32$, and CIFAR-10 FID
reference.  The Euler rows in Table~\ref{tab:edm-cifar-equal-nfe-runtime} differ
only in timestep placement.  DPM++ uses a second-order EDM multistep update,
whereas the original EDM configuration uses a Heun-style correction with $18$
steps, i.e. $35$ network evaluations.

\begin{keybox}
\textbf{Two distinct comparisons.}
The direct test of the proposed method is Euler against Euler at fixed NFE:
the learned field and numerical update are unchanged, and only the nodes of the
grid move.  DPM++ and EDM Heun instead modify the integration formula.  They
calibrate the accuracy of stronger solvers but do not test the timestep
allocation claim.  A density-active profile can also be used as a diagnostic or
combined with a different solver; this paper isolates it with Euler so that the
allocation effect is identifiable.
\end{keybox}

\begin{figure}[H]
\centering
\includegraphics[width=0.74\linewidth]{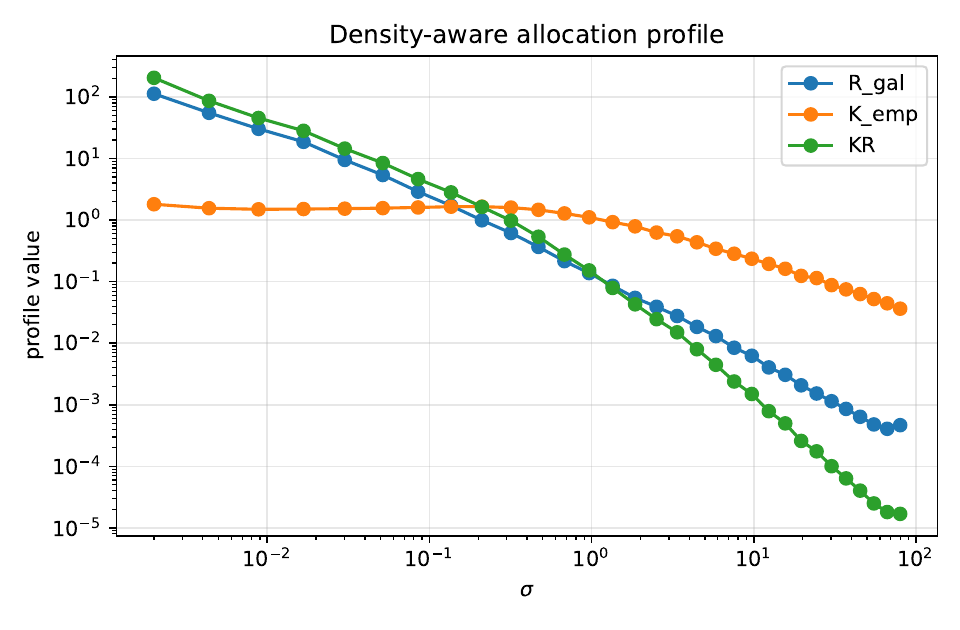}
\caption{Offline profile used by the Euler allocation: the Galerkin estimate of
the density-active residual, the empirical terminal-propagation factor, and
their product.  The schedule is obtained from quantiles of the square root of
the propagated profile in the solver variable $\sigma$.}
\label{fig:app-edm-cifar-profile}
\end{figure}

The profile assigns more of the Euler budget to the low-noise region without
collapsing all evaluations at the endpoint.  It is estimated once and then
reused for every image seed and every displayed Euler budget.  Consequently,
the three Euler rows have essentially identical generation time at fixed NFE.

\paragraph{Dense-grid convergence and asymptotic target.}
To distinguish faster convergence from a changed high-NFE target, we compare
matched-seed endpoints against a dense numerical reference.  For each of $1024$
shared noise seeds and class labels, the reference is a $1000$-step Euler
trajectory on the EDM-$\rho$ grid.  We then integrate the same initial condition
with each Euler schedule at $8,16,20,35,64,128,256,$ and $512$ NFE.  Choosing an
EDM-$\rho$ reference is conservative for our method, yet the density-aware grid
has the smallest endpoint RMSE at every tested budget.

\begin{figure}[H]
\centering
\includegraphics[width=0.78\linewidth]{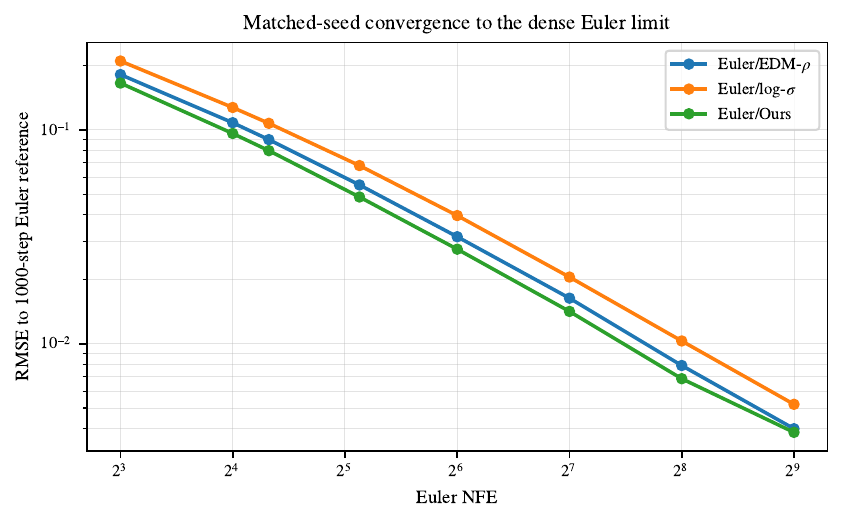}
\caption{Matched-seed convergence to a $1000$-step Euler reference.  The proposed
schedule is closest to the dense reference at every tested budget, while the
three curves approach one another as NFE grows.  The reference is numerical,
not an exact ODE solution.}
\label{fig:app-edm-dense-euler}
\end{figure}

The endpoint result is complemented by an independent FID-$5$k run at $1000$
NFE.  The three values differ by less than $0.002$, while generation times stay
within a few percent.  Thus the low-NFE improvement does not create a different
asymptotic sampler; it accelerates a fixed Euler update toward the same dense-grid
regime.

\begin{table}[H]
\caption{High-NFE Euler check on $5000$ CIFAR-10 samples.  Every row uses
$1000$ Euler evaluations; FID computation time is excluded.}
\label{tab:app-edm-1000nfe}
\centering
\begingroup
\small
\setlength{\tabcolsep}{6pt}
\begin{tabular}{lcc}
\toprule
Schedule & FID-$5$k & Generation time (s) \\
\midrule
Euler/Ours & 6.64496 & 5187.9 \\
Euler/log-$\sigma$ & 6.64442 & 5112.5 \\
Euler/EDM-$\rho$ & 6.64618 & 5289.6 \\
\bottomrule
\end{tabular}
\endgroup
\end{table}

\begin{figure}[H]
\centering
\includegraphics[width=0.74\linewidth]{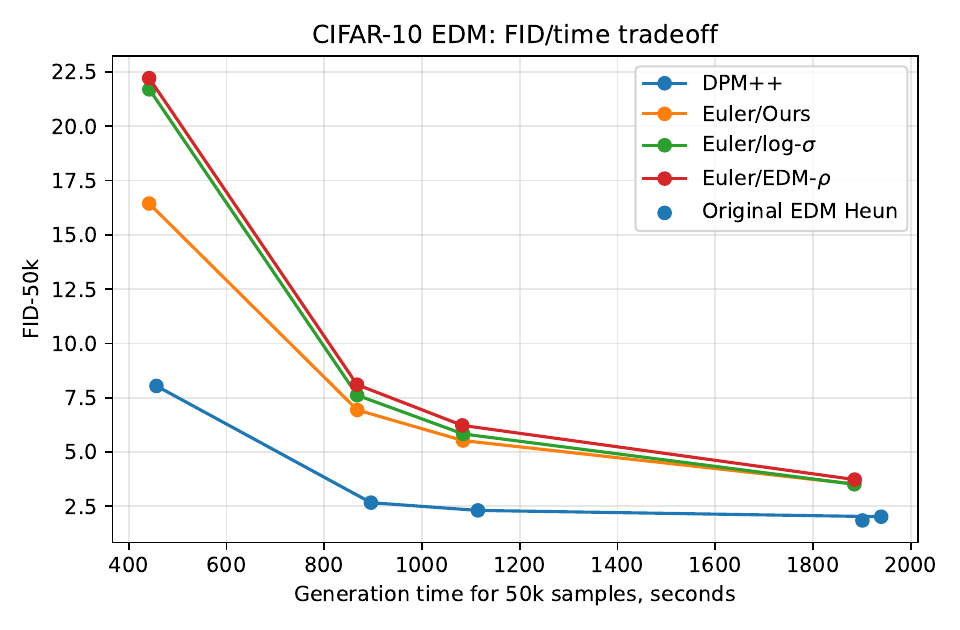}
\caption{FID-50k versus wall-clock generation time.  The Euler schedules lie at
nearly identical cost for each NFE.  DPM++ and EDM Heun reach lower FID by using
stronger update formulas, with a small implementation-dependent overhead for
DPM++.}
\label{fig:app-edm-cifar-tradeoff}
\end{figure}

\begin{figure}[H]
\centering
\begingroup
\small
\setlength{\tabcolsep}{5pt}
\renewcommand{\arraystretch}{1.08}
\begin{tabular}{@{}p{0.21\linewidth}p{0.75\linewidth}@{}}
\shortstack[l]{\textbf{Euler/EDM-$\rho$}\\[-1mm]{\footnotesize FID-50k 8.10}} &
\includegraphics[width=\linewidth]{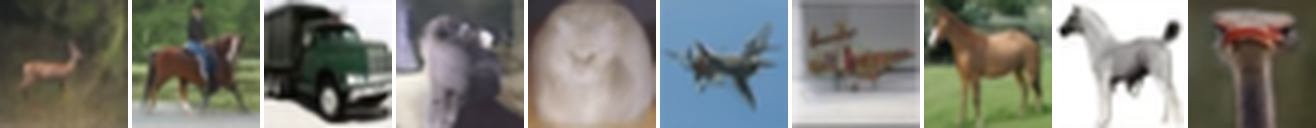}\\
\shortstack[l]{\textbf{Euler/log-$\sigma$}\\[-1mm]{\footnotesize FID-50k 7.62}} &
\includegraphics[width=\linewidth]{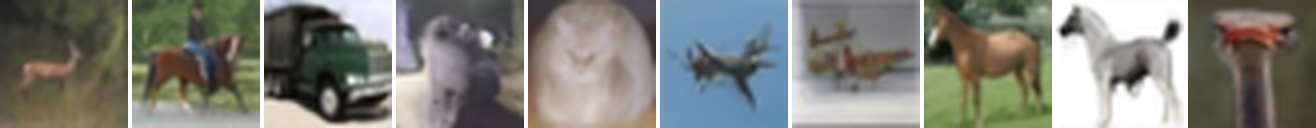}\\
\shortstack[l]{\textbf{Euler/Ours}\\[-1mm]{\footnotesize FID-50k 6.94}} &
\includegraphics[width=\linewidth]{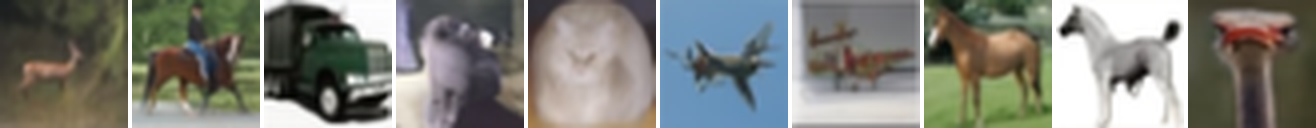}
\end{tabular}
\endgroup
\caption{Matched-seed CIFAR-10 samples at 16 NFE for the three Euler schedules.
As in the main-text figure, the comparison changes only timestep placement.}
\label{fig:app-edm-cifar-samples}
\end{figure}

\paragraph{Numerical protocol.}
The residual profile is estimated from $256$ trajectories on $31$ noise levels.
A $128$-step Heun trajectory supplies the fine reference, and a random-feature
Galerkin solve approximates the weighted-gradient projection.  The terminal
factor is estimated by propagating small residual-aligned perturbations with the
same fine solver.  All main-table FID values use $50{,}000$ generated images; generation
times in Table~\ref{tab:edm-cifar-equal-nfe-runtime} exclude FID computation.
The dense-limit endpoint study uses $1024$ matched samples, and
Table~\ref{tab:app-edm-1000nfe} uses an independent $5000$-sample FID run.

%% file: appendices/D_analytic_synthetic_experiments.tex
\section{Exact Mechanism Checks}
\label{app:synthetic-experiments}
\label{sec:synthetic-experiments}

The following analytic examples verify the mechanisms that motivate the
allocation rule.  They are deliberately small: the relevant marginals,
velocities, and density-active defects are available in closed form, so the
numerical observation can be compared directly with the theorem rather than
with a learned proxy.

\subsection{Gaussian cancellation and parity}
Consider
\begin{equation}
    p_t=\mathcal N(0,s(t)^2I_2),
    \qquad s(t)=1+t^2,
\end{equation}
and the feasible linear gauges
\begin{equation}
    w_t^\alpha(x)=\beta(t)x+\alpha\omega(t)Jx,
    \qquad
    \beta(t)=\frac{\dot s(t)}{s(t)},
    \qquad
    \omega(t)=\sqrt{\frac{\ddot s(t)}{s(t)}}.
    \label{eq:alpha-gauge}
\end{equation}
Here $J^\top=-J$ and $J^2=-I_2$.  The density-active acceleration coefficient is
$\frac{\ddot s(t)}{s(t)}(1-\alpha^2)$, hence the continuum bias functional
vanishes exactly at $\alpha=\pm1$.

\begin{figure}[!htbp]
\centering
\begin{minipage}[t]{0.48\linewidth}
\centering
\includegraphics[width=0.98\linewidth]{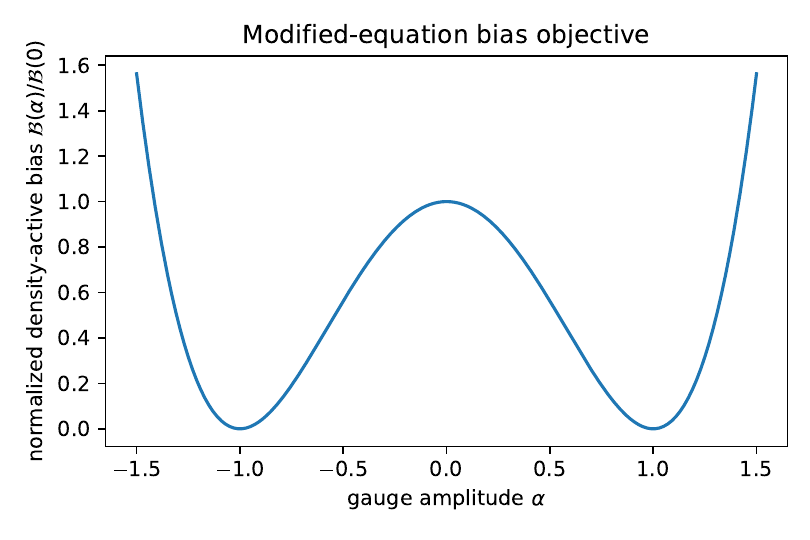}
\end{minipage}\hfill
\begin{minipage}[t]{0.48\linewidth}
\centering
\includegraphics[width=0.98\linewidth]{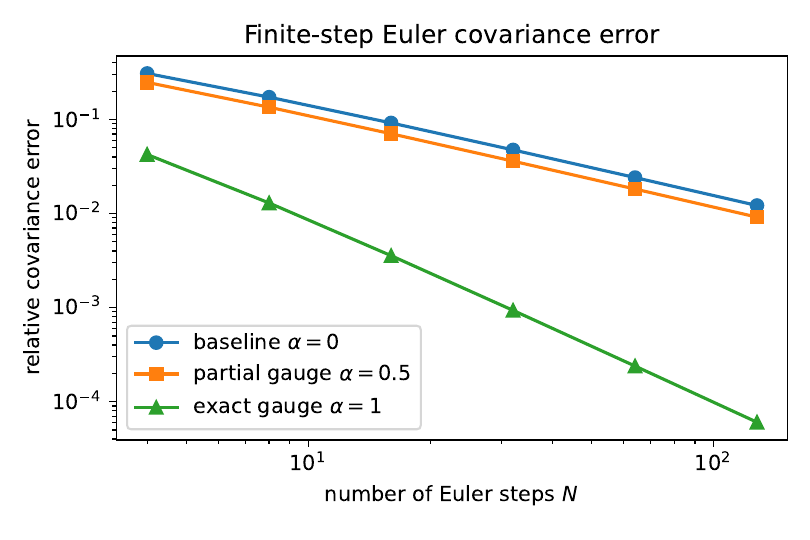}
\end{minipage}
\caption{Exact Gaussian gauge check.  Left: normalized density-active
modified-equation bias, with zeros at the predicted gauges $\alpha=\pm1$.
Right: finite-step Euler covariance error for the same continuum density path;
the cancelling skew gauge reduces the error at every displayed step count.}
\label{fig:synthetic}
\end{figure}

\begin{table}[!htbp]
\caption{Relative covariance error for the Gaussian path $s(t)=1+t^2$ under
Euler integration.}
\label{tab:gaussian-errors}
\centering
  \begingroup
  \small
  \setlength{\tabcolsep}{4.2pt}%
  \renewcommand{\arraystretch}{0.96}%
  \begin{adjustbox}{max width=0.58\linewidth,center}
    \input{figures/gaussian_errors_table.tex}  \end{adjustbox}%
  \endgroup

\end{table}

The same construction yields the parity obstruction: an orthogonal complex
structure can cancel the isotropic residual in even dimension, whereas one
scalar direction remains in odd dimension.  Direct minimization of the linear
residual gives zero to numerical precision in dimensions $2$ and $4$ and a
strictly positive residual in dimensions $3$ and $5$, matching
Theorem~\ref{thm:parity}.

\subsection{Why full material acceleration is not the allocation signal}
A time-dependent skew gauge can rotate rapidly while leaving the leading
marginal covariance defect equal to zero.  In the controlled four-dimensional
construction, the maximum density-active residual is approximately
$5\times10^{-15}$, whereas the full material-acceleration proxy reaches about
$136$.  Thus a schedule based on the full acceleration would concentrate nodes
around a large but marginal-invisible spike.  The example is an exact
counterexample to the rule ``more trajectory curvature always means more
marginal error''; no performance claim relies on the visual appearance of a
particular tick plot.

\begin{takeawaybox}
\textbf{Mechanism check.}
The Gaussian examples establish two points used later: a marginal-preserving
gauge can cancel the leading density error of a finite solver, and a large
trajectory-level acceleration can remain invisible to the marginal.  The
hard-GMM and CIFAR-10 studies then evaluate timestep allocation in nontrivial
distributional settings.
\end{takeawaybox}

%% file: figures/gaussian_errors_table.tex
\begin{tabular}{rrrr}
\toprule
$N$ & $\alpha=0$ & $\alpha=0.5$ & $\alpha=1$\\
\midrule
4 & 3.0856e-01 & 2.4783e-01 & 4.2401e-02\\
8 & 1.7267e-01 & 1.3500e-01 & 1.2940e-02\\
16 & 9.1856e-02 & 7.0496e-02 & 3.5626e-03\\
32 & 4.7453e-02 & 3.6023e-02 & 9.3324e-04\\
64 & 2.4128e-02 & 1.8209e-02 & 2.3871e-04\\
128 & 1.2167e-02 & 9.1540e-03 & 6.0355e-05\\
\bottomrule
\end{tabular}

%% file: appendices/E_hard_gmm_experimental_details.tex
\section{Hard-GMM Allocation Study}
\label{app:controlled-extra}

This appendix tests the scheduling consequence on a controlled eight-mode
Gaussian-mixture path with unequal weights, rotating and separating centers,
and decreasing component variance.  Unlike the exact Gaussian checks, the
proxies are correlated rather than artificially separated; the experiment asks
whether the theory-linked density-active schedule remains competitive under a
fixed low-NFE budget.

\subsection{Primary protocol}
\label{app:experimental-protocol}
The profile is evaluated at $31$ times using $1536$ particles per time.
Schedules are constructed from the density-active residual, the full material-
acceleration norm, or the velocity norm.  The allocation exponent is either
fixed at the theoretical Euler value $\gamma=1/2$ or selected on disjoint
validation seeds.  Final metrics use independent test seeds.

\begin{table}[!htbp]
\caption{Primary hard-GMM evaluation protocol.}
\label{tab:protocol-proxy-ablation}
\centering
\small
\setlength{\tabcolsep}{5pt}
\begin{tabular}{p{0.30\linewidth}p{0.54\linewidth}}
\toprule
Quantity & Value \\
\midrule
Profile grid & $31$ points in $[0,1]$, $1536$ particles per time \\
Sampling budget & $2048$ particles per evaluation seed \\
Validation/test split & validation selects $\gamma$; test seeds are disjoint \\
MMD & RBF MMD with deterministic truncation to $512$ points \\
Sliced Wasserstein & $128$ random directions \\
Schedule family & uniform or proxy quantiles with $\gamma\in\{0.25,0.5,0.75,1\}$ \\
Fair NFE convention & one Euler interval costs one field evaluation; one Heun interval costs two \\
\bottomrule
\end{tabular}
\normalsize
\end{table}

\subsection{Theory-linked fixed-exponent comparison}
The three profiles are
\[
\begin{aligned}
R_{\rm dens}(t)
&=\left(\mathbb E_{p_t}\left[(\nabla\cdot a_v+
\langle\nabla\log p_t,a_v\rangle)^2\right]\right)^{1/2},\\
R_{\rm acc}(t)&=\left(\mathbb E_{p_t}\|a_v\|^2\right)^{1/2},
\qquad
R_{\rm speed}(t)=\left(\mathbb E_{p_t}\|v_t\|^2\right)^{1/2}.
\end{aligned}
\]
Legacy filenames call $R_{\rm acc}$ ``curvature,'' but the measured quantity is
the full material-acceleration norm.  On this path all three proxies increase
near the difficult late-time region, so their schedules are expected to be
similar; the exact Gaussian example in Appendix~\ref{app:synthetic-experiments}
is the selective counterexample where they separate.

\begin{table}[!htbp]
\caption{Theory-linked hard-GMM comparison with the fixed Euler exponent
$\gamma=1/2$.}
\label{tab:hard-gmm-fixed-gamma}
\centering
  \begingroup
  \small
  \setlength{\tabcolsep}{4.2pt}%
  \renewcommand{\arraystretch}{0.96}%
  \begin{adjustbox}{max width=0.88\linewidth,center}
    \input{figures/hard_gmm_fixed_gamma_table.tex}  \end{adjustbox}%
  \endgroup

\end{table}

\begin{figure}[!htbp]
\centering
\includegraphics[width=0.72\linewidth]{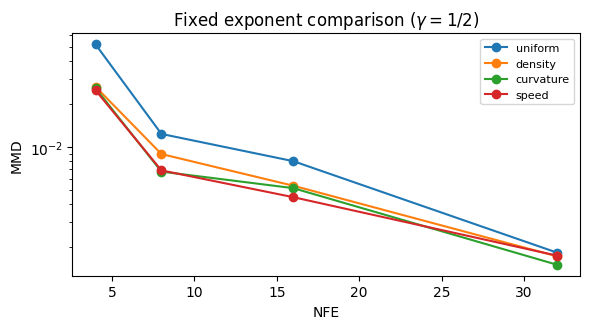}
\caption{Terminal MMD at fixed $\gamma=1/2$.  Every nonuniform schedule improves
on uniform Euler.  The close values of the three proxies reflect their shared
late-time concentration on this benchmark, rather than equivalence of the
underlying error signals.}
\label{fig:hard-gmm-fixed-gamma}
\end{figure}

\subsection{Validation-selected and mixed-order baselines}
Appendix~\ref{app:unified-new-experiments} reports the disjoint
validation/test comparison against embedded Euler--Heun refinement, an
Jo--Choi velocity-variation schedule, a particle-only dual residual,
directional propagation weighting, a trajectory oracle, and fair-NFE Heun.
A separate hybrid experiment spends Heun corrections only on intervals selected
by a proxy; the density-active signal can therefore be used either to move
Euler nodes or to allocate a mixed-order budget.

\begin{table}[!htbp]
\caption{Hybrid Euler/Heun allocation under a fixed NFE budget.  On this
correlated benchmark the three nonuniform selectors choose nearly the same
late-time intervals and all improve over uniform selection.}
\label{tab:hard-gmm-hybrid}
\centering
  \begingroup
  \small
  \setlength{\tabcolsep}{4.2pt}%
  \renewcommand{\arraystretch}{0.96}%
  \begin{adjustbox}{max width=0.76\linewidth,center}
    \input{figures/hard_gmm_hybrid_table.tex}  \end{adjustbox}%
  \endgroup

\end{table}

\begin{takeawaybox}
\textbf{Hard-GMM takeaway.}
The density-active schedule improves over uniform Euler at the theory-predicted
exponent and remains competitive with simpler geometric proxies when the
proxies are strongly correlated.  The purpose of the exact distractor is to
show why the density projection is the principled signal; the purpose of this
mixture study is to test that the resulting schedule remains useful in a hard,
non-Gaussian path.
\end{takeawaybox}

%% file: figures/hard_gmm_fixed_gamma_table.tex
\begin{tabular}{@{}clcrrrr@{}}
\toprule
NFE & Proxy & $\gamma$ & MMD $\downarrow$ & SW $\downarrow$ & Wt. err. $\downarrow$ & Leak. $\downarrow$ \\
\midrule
8  & Uniform   & 0.00 & 0.01235 & 0.7819 & 0.1196 & 1.0000 \\
8  & Density   & 0.50 & 0.008942 & 0.5546 & 0.2264 & 0.8704 \\
8  & Mat. acc. & 0.50 & 0.006744 & 0.5673 & 0.1221 & 0.9993 \\
8  & Speed     & 0.50 & 0.006880 & 0.5749 & 0.1517 & 0.9966 \\
16 & Uniform   & 0.00 & 0.007978 & 0.4685 & 0.0857 & 0.9764 \\
16 & Density   & 0.50 & 0.005379 & 0.3247 & 0.1178 & 0.2305 \\
16 & Mat. acc. & 0.50 & 0.005172 & 0.3463 & 0.0891 & 0.5260 \\
16 & Speed     & 0.50 & 0.004475 & 0.3376 & 0.0929 & 0.4727 \\
\bottomrule
\end{tabular}

%% file: figures/hard_gmm_hybrid_table.tex
\begin{tabular}{@{}clcccc@{}}
\toprule
NFE & Selector & $K$ & $q$ & MMD $\downarrow$ & SW $\downarrow$ \\
\midrule
8  & Uniform   & 5  & 3  & 0.01636 & 0.6496 \\
8  & Density   & 5  & 3  & 0.003201 & 0.3333 \\
8  & Mat. acc. & 5  & 3  & 0.003201 & 0.3333 \\
8  & Speed     & 5  & 3  & 0.003201 & 0.3333 \\
16 & Uniform   & 10 & 6  & 0.001887 & 0.2866 \\
16 & Density   & 10 & 6  & 0.001766 & 0.1459 \\
16 & Mat. acc. & 10 & 6  & 0.001766 & 0.1459 \\
16 & Speed     & 10 & 6  & 0.001766 & 0.1459 \\
32 & Uniform   & 21 & 11 & 0.000455 & 0.2053 \\
32 & Density   & 21 & 11 & 0.000386 & 0.1311 \\
32 & Mat. acc. & 21 & 11 & 0.000386 & 0.1311 \\
32 & Speed     & 21 & 11 & 0.000386 & 0.1311 \\
\bottomrule
\end{tabular}

%% file: appendices/M_unified_baselines_and_estimator_validation.tex
\section{Unified Additional Baselines and Estimator Validation}
\label{app:unified-new-experiments}

Every method below is evaluated on the same hard-GMM path, with the same metric implementation and independent validation/test seeds.  The Jo--Choi baseline implements the fixed-budget velocity-variation resampling principle of \citet{Jo2026SamplingDesign} on the same analytic path and evaluation protocol.

\subsection{Protocol}
The analytic eight-mode path and evaluation metrics are unchanged.  Residual profiles use $31$ time points and $1536$ particles per point.  Four validation seeds choose $\gamma$ from
$\{0.15,0.25,0.35,0.5,0.75,1\}$; the selected schedule is then evaluated on eight disjoint test seeds.  Each test seed uses $1536$ generated and target samples, an RBF MMD computed from $384$ points, and $64$ sliced-Wasserstein directions.  Tables report mean $\pm$ standard error over the eight test seeds.

\begin{keybox}
\textbf{Compared schedules.}
\begin{enumerate}
\item \textbf{Density, acceleration, and speed with validation selection.}  The exponent is selected only on validation seeds and frozen before test evaluation.
\item \textbf{Embedded Euler--Heun AMR.}  Starting from $[0,1]$, the interval with largest pilot RMS Euler--Heun discrepancy is bisected until the requested number of Euler intervals is reached.
\item \textbf{Jo--Choi velocity variation.}  On a $128$-interval pilot trajectory we estimate
\[
\eta_i=\left(\mathbb E\|v_{t_{i+1}}(X_{t_{i+1}})-v_{t_i}(X_{t_i})\|^2\right)^{1/2}.
\]
The fixed-NFE grid is uniform in the cumulative length $\sum_i\eta_i^{1/2}$, matching the square-root resampling power used by the official N-step resampling construction of \citet{Jo2026SamplingDesign}.
\item \textbf{Particle dual.}  The residual identity~(\ref{eq:dual-particle-residual}) is optimized in a finite critic space consisting of $64$ Gaussian RBFs and low-order polynomial features.  This estimator uses particles and material acceleration, but no score and no explicit divergence.
\item \textbf{Directional-$K$ density.}  A directional tangent-amplification proxy is estimated by perturbing particles along $a_v$ and propagating both copies to $t=1$ with fine RK4.  Because raw amplification is heavy-tailed near responsibility boundaries, we take the median over three pilot seeds, apply a width-three median filter in log scale, and clip only to its empirical $5$--$95\%$ range.  This is explicitly a robust proxy for $K(1,t)$, not an operator-norm estimate.
\item \textbf{Coupling oracle.}  A differentiable grid minimizes samplewise squared deviation from a $256$-step RK4 reference on separate training particles.  It is a trajectory-coupling oracle and need not minimize MMD.
\item \textbf{Fair-NFE Heun.}  A uniform Heun solver uses $N/2$ intervals under an $N$-evaluation budget.
\end{enumerate}
\end{keybox}

\subsection{Unified results}
\begin{table}[!htbp]
\caption{Unified hard-GMM test results after validation-only hyperparameter selection.  The coupling oracle is optimized for trajectory MSE, not for either displayed distribution metric.}
\label{tab:unified-baselines-full}
\centering
  \begingroup
  \small
  \setlength{\tabcolsep}{4.2pt}%
  \renewcommand{\arraystretch}{0.96}%
  \begin{adjustbox}{max width=0.84\linewidth,center}
    \input{figures/unified_baselines_table.tex}  \end{adjustbox}%
  \endgroup

\end{table}

\begin{figure}[!htbp]
\centering
\includegraphics[width=0.73\textwidth]{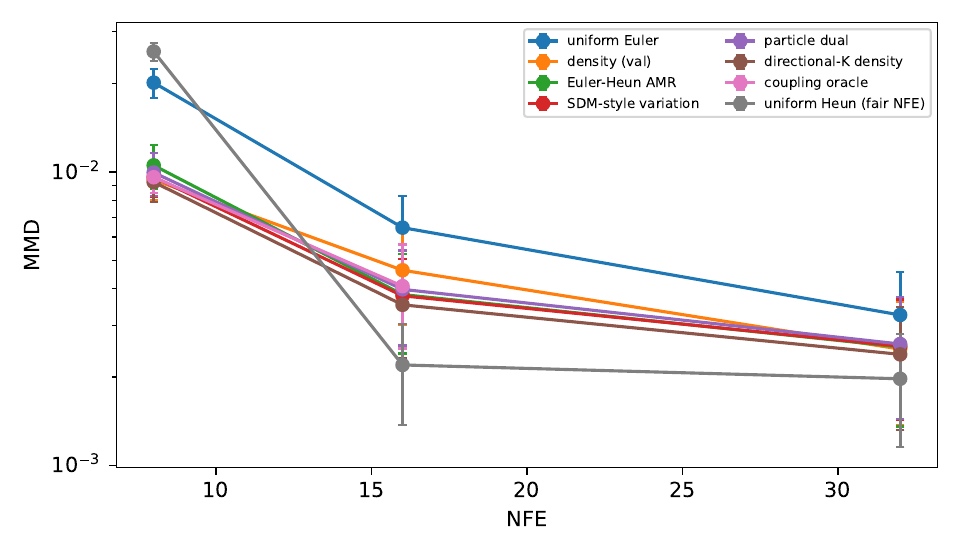}
\caption{MMD under the unified protocol.  Density-active schedules remain competitive with embedded Euler--Heun refinement and Jo--Choi velocity-variation resampling.  Fair-NFE Heun illustrates the complementary benefit of higher-order integration as the budget supports more second-order intervals.}
\label{fig:unified-baselines-mmd}
\end{figure}

The validation-selected density schedule reduces MMD from $0.02013\pm0.00232$ to $0.00939\pm0.00138$ at NFE $8$, and from $0.00646\pm0.00184$ to $0.00462\pm0.00161$ at NFE $16$.  The Euler--Heun AMR and Jo--Choi schedules are comparable: at NFE $8$ their MMD values are $0.01052\pm0.00179$ and $0.00960\pm0.00140$, respectively.  Thus the density-active signal is not merely beating a uniform grid; it lies on the same controlled Pareto scale as two solver-error or Wasserstein-variation alternatives.

The robust directional-$K$ weighting gives the lowest displayed Euler MMD at NFE $8$ and $16$, namely $0.00923\pm0.00134$ and $0.00353\pm0.00121$.  Its close agreement with the residual-only schedule supports propagation weighting as a refinement, while the robust median filtering stabilizes the directional sensitivity profile near responsibility boundaries.

The trajectory-coupling oracle reaches MMD $0.00959\pm0.00114$ at NFE $8$ and $0.00408\pm0.00158$ at NFE $16$.  It achieves comparable terminal MMD while minimizing a distinct paired-trajectory objective.  This makes the trajectory-versus-density distinction operational in a second way, complementary to the solenoidal distractor.

\subsection{Particle-only residual}
\begin{figure}[!htbp]
\centering
\includegraphics[width=0.73\textwidth]{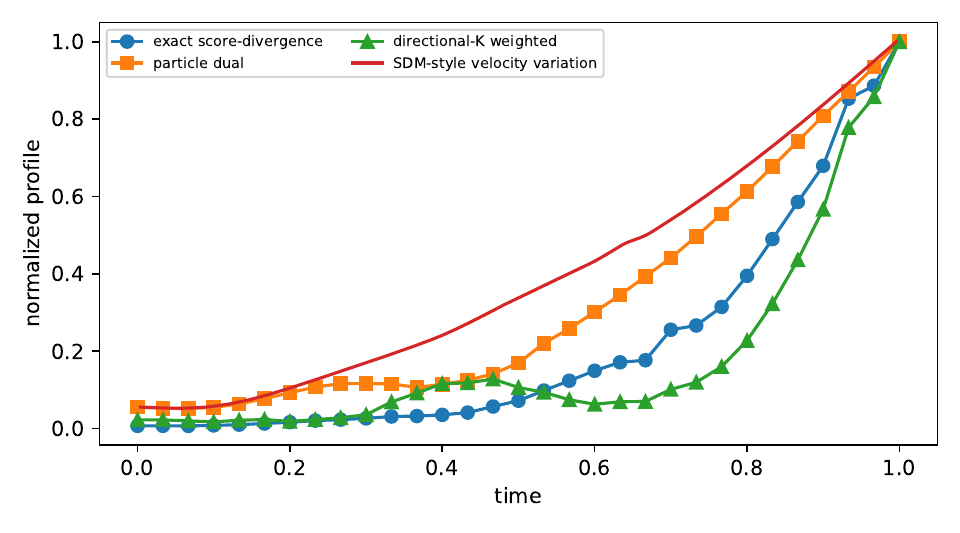}
\caption{Exact score--divergence, particle-only dual, robust directional-$K$-weighted, and Jo--Choi velocity-variation profiles, normalized independently.}
\label{fig:profile-estimators}
\end{figure}

Across the $31$ profile times, the finite RBF dual estimator has Pearson correlation $0.9760$ and Spearman correlation $0.9907$ with the exact score--divergence profile.  Its fixed-$\gamma=1/2$ schedule yields MMD $0.00995\pm0.00166$ at NFE $8$ and $0.00398\pm0.00142$ at NFE $16$.  Hence the unknown-density route in the dual residual identity~(\ref{eq:dual-particle-residual}) is not only formal: in this controlled experiment it recovers the ordering of difficult time regions and produces a competitive grid without access to $\nabla\log p_t$.

\paragraph{Profile cost and amortization.}
On the reported CPU environment, a $31\times1024$ exact score--divergence
profile takes $1.21$ seconds and the finite particle-dual profile takes $2.76$
seconds.  For a batch of $1536$ particles, uniform Euler at NFE $16$ takes
$0.0310$ seconds, while density-scheduled Euler at NFE $8$ takes $0.0159$
seconds.  Under this implementation, amortizing the exact profile requires
about $80$ batches ($1.22\times10^5$ generated samples), while the particle-dual
profile requires about $182$ batches ($2.80\times10^5$ samples).  These values
are implementation-specific and are included only to make the offline/online
cost separation explicit.

\begin{takeawaybox}
\textbf{Unified evidence.}  Hyperparameters are chosen on disjoint validation seeds; density scheduling is compared against embedded local-error refinement and an Jo--Choi Wasserstein-variation schedule; the particle-only dual residual is empirically validated; and profile cost is converted into a concrete amortization threshold.  The protocol therefore combines tuning, baseline comparison, score-free residual estimation, and explicit profile-cost amortization in one controlled evaluation.
\end{takeawaybox}

%% file: figures/unified_baselines_table.tex
\begin{tabular}{@{}clcc@{}}
\toprule
NFE & Method & MMD $\downarrow$ & SW $\downarrow$ \\
\midrule
8 & Uniform Euler & 0.0201 $\pm$ 0.002 & 0.7987 $\pm$ 0.017 \\
8 & Density (val.) & 0.0094 $\pm$ 0.001 & 0.5528 $\pm$ 0.016 \\
8 & Acceleration (val.) & 0.0096 $\pm$ 0.001 & 0.5883 $\pm$ 0.015 \\
8 & Speed (val.) & 0.0095 $\pm$ 0.001 & 0.5867 $\pm$ 0.015 \\
8 & Euler--Heun AMR & 0.0105 $\pm$ 0.002 & 0.5469 $\pm$ 0.019 \\
8 & SDM variation & 0.0096 $\pm$ 0.001 & 0.5811 $\pm$ 0.022 \\
8 & Particle dual & 0.0099 $\pm$ 0.002 & 0.5589 $\pm$ 0.020 \\
8 & Directional-$K$ & 0.0092 $\pm$ 0.001 & 0.5225 $\pm$ 0.024 \\
8 & Coupling oracle & 0.0096 $\pm$ 0.001 & 0.5523 $\pm$ 0.022 \\
8 & Uniform Heun & 0.0257 $\pm$ 0.002 & 0.8937 $\pm$ 0.015 \\
\midrule
16 & Uniform Euler & 0.0065 $\pm$ 0.002 & 0.4305 $\pm$ 0.012 \\
16 & Density (val.) & 0.0046 $\pm$ 0.002 & 0.3509 $\pm$ 0.011 \\
16 & Acceleration (val.) & 0.0036 $\pm$ 0.0009 & 0.3015 $\pm$ 0.017 \\
16 & Speed (val.) & 0.0037 $\pm$ 0.0010 & 0.3155 $\pm$ 0.015 \\
16 & Euler--Heun AMR & 0.0038 $\pm$ 0.001 & 0.3079 $\pm$ 0.016 \\
16 & SDM variation & 0.0038 $\pm$ 0.001 & 0.3130 $\pm$ 0.013 \\
16 & Particle dual & 0.0040 $\pm$ 0.001 & 0.3145 $\pm$ 0.014 \\
16 & Directional-$K$ & 0.0035 $\pm$ 0.001 & 0.2865 $\pm$ 0.015 \\
16 & Coupling oracle & 0.0041 $\pm$ 0.002 & 0.3119 $\pm$ 0.012 \\
16 & Uniform Heun & 0.0022 $\pm$ 0.0008 & 0.1651 $\pm$ 0.012 \\
\bottomrule
\end{tabular}

%% file: appendices/P_theorem_linked_experiments.tex
\section{Additional Theorem-Linked Quantitative Checks}
\label{app:m1-experiments}

This appendix collects compact numerical checks whose quantities are known
analytically or can be evaluated against an independently optimized oracle.
They support the local-order, propagation, and robustness statements without
serving as the main empirical benchmark; the pretrained CIFAR-10 result remains
the primary image-model validation.

\subsection{Analytic mixture path and objective specificity}
The two-dimensional path contains eight unequal-weight Gaussian components whose
centers rotate and separate while their common variance contracts.
Responsibilities give an analytic global velocity satisfying the continuity
equation; the normalized finite-difference continuity residual is below
$10^{-8}$ RMS.  At $N=8$, four independently constructed grids are compared on
held-out particles.

\begin{table}[!htbp]
\centering
\caption{Analytic GMM grid comparison at $N=8$.  MMD is mean $\pm$ standard
error over three evaluation seeds.  The density-quantile grid exactly minimizes
the local surrogate, while the terminal and trajectory oracles optimize
different objectives.}
\label{tab:m1-oracle-geometry}
\small
\setlength{\tabcolsep}{4.5pt}
\begin{tabular}{@{}lccc@{}}
\toprule
Grid & Test MMD & Local surrogate & $d(\pi,\pi_{\rm marg})$\\
\midrule
Uniform & $0.0400\pm0.0094$ & $1.184$ & $0.0858$\\
Density quantile & $0.0363\pm0.0175$ & $\mathbf{1.000}$ & $0.1412$\\
Trajectory oracle & $0.0357\pm0.0153$ & $1.040$ & $0.1272$\\
Marginal RFF oracle & $0.0381\pm0.0107$ & $2.377$ & $0$\\
\bottomrule
\end{tabular}
\normalsize
\end{table}

The terminal marginal oracle minimizes a fixed random-feature MMD on separate
optimization particles.  Its local-surrogate value is therefore not expected
to be optimal.  Conversely, the density grid is closer to the trajectory oracle
than to the terminal-feature oracle.  The result makes the paper's distinction
operational: local marginal error, paired-trajectory error, and a terminal
feature metric can induce different timestep geometries.

\subsection{Order, propagation, and profile robustness}
Table~\ref{tab:m1-theorem-checks} summarizes the theorem-linked checks in a compact numerical form.  Reporting the numerical quantities
directly avoids over-interpreting finite collections of points as performance
curves.

\begin{table}[!htbp]
\centering
\caption{Compact theorem-linked checks.  The propagation row reports terminal
$W_2$ at $N=16$; the robustness row compares the coarsest and finest grids.}
\label{tab:m1-theorem-checks}
\small
\setlength{\tabcolsep}{4.5pt}
\begin{tabular}{p{0.25\linewidth}p{0.28\linewidth}p{0.36\linewidth}}
\toprule
Claim & Prediction & Observed check\\
\midrule
Local marginal order & radial order $2$; cancelled even-dimensional skew order $3$; odd-dimensional obstruction order $2$
& fitted tail slopes $2.000$, $2.998$--$3.002$, and $2.000$--$2.00004$, respectively\\
Propagation weighting & $K(1,t)R(t)$ approaches the terminal-grid oracle
& at $N=16$: uniform $0.18544$, residual-only $0.19615$, $K$-weighted $0.18514$, oracle $0.18509$\\
Profile robustness & the continuous $\chi^2$ identity is recovered as the quantile grid is refined
& correlation increases from $0.891$ at $N=8$ to $0.995$ at $N=64$; median relative discrepancy decreases from $0.751$ to $0.125$\\
\bottomrule
\end{tabular}
\normalsize
\end{table}

The residual-only schedule can be worse on the anisotropic propagation example
because it ignores how a local perturbation is amplified before the terminal
time.  The $K$-weighted grid nearly matches the direct oracle, which motivates
the empirical terminal factor used in the CIFAR-10 experiment.  The robustness
check also supports the square-root allocation: under multiplicative profile
noise, it perturbs the grid less than the linear-power alternative.

%% file: appendices/H_applications_and_optimization_relevance.tex
\section{Applications and Optimization Relevance}
\label{app:applications}

The framework has three optimization-oriented applications.

\begin{takeawaybox}
\textbf{Adaptive low-NFE schedules.}  The most robust application is timestep optimization.  The density-active residual $R(t)$ is derived from a marginal backward-error identity, rather than from a purely trajectory-level local error.  Theorem~\ref{thm:propagated-defect} shows that the terminally relevant profile is $K(1,t)R_{\rm loc}(t)$, and Theorem~\ref{thm:schedule-opt} gives the corresponding constrained allocation.  The controlled residual-only ablations set $K\equiv1$, while the pretrained
EDM experiment estimates a reusable empirical propagation factor $K_{\rm emp}$.
\end{takeawaybox}

\paragraph{Solver diagnostics.}
The residual
\[
\nabla\cdot(p_ta_v)
=p_t\bigl(\nabla\cdot a_v+\langle\nabla\log p_t,a_v\rangle\bigr)
\]
identifies where a finite-step solver creates density-level bias.  This can be used independently of any velocity correction: high-residual intervals should receive more steps or a higher-order local solver.  Conversely, a high full material-acceleration interval should not automatically be treated as important unless its acceleration is density-active; the Gaussian curvature-distractor experiment gives an explicit counterexample.

\paragraph{Gauge-aware sampler design.}
The gauge theory complements the schedule layer by characterizing numerically favorable representatives.  It explains why marginal-preserving perturbations can be harmful after discretization and gives a trust-region formulation for safe corrections:
\[
\min_{u\in\cU_p}\cB(v+u)
\quad\text{s.t.}\quad
\frac{\int\E_{p_t}\|u_t\|^2dt}{\int\E_{p_t}\|v_t\|^2dt}\le \rho^2.
\]
The controlled experiments suggest that this layer should be conservative and combined with line search or validation, rather than used as an unrestricted learned correction.

\paragraph{Diversity with numerical-bias control.}
Divergence-free particle guidance can increase diversity, but it can also increase finite-step density bias.  The constrained formulation allows one to add diversity-promoting gauges under a side constraint on $\cB(v+u)$:
\[
\max_{u\in\cU_p}\operatorname{Diversity}(v+u)
\quad\text{s.t.}\quad
\cB(v+u)\le \cB(v)+\varepsilon .
\]
This separates continuum marginal preservation from numerical harmlessness.

%% file: appendices/N_local_optimal_transport_interpretation.tex
\section{Local Optimal-Transport Interpretation}
\label{app:local-ot}

This appendix proves the Euler specialization of Theorem~\ref{thm:local-w2-defect} and the isotropic Gaussian separation stated in Section~\ref{sec:gauge-consequences-main}; Appendix~\ref{app:solver-aware-limits} gives the order-$r$ argument.  We work on a compact boundaryless domain or on $\mathbb R^d$ with sufficient decay and finite second moments.  Assume that $p_t$ is smooth and strictly positive, $w_t$ has the derivatives required for the Taylor expansions below, and the involved density perturbations lie in a regular Wasserstein chart around $p_t$.  The standard tangent characterization of $W_2$ identifies a perturbation
\begin{equation}
\rho_\varepsilon=\rho-\varepsilon\nabla\cdot(\rho\xi)+o(\varepsilon)
\end{equation}
with tangent norm
\begin{equation}
W_2(\rho_\varepsilon,\rho)
=|\varepsilon|\,\|P_{\nabla,\rho}\xi\|_{L^2(\rho)}+o(|\varepsilon|),
\label{eq:w2-tangent-chart}
\end{equation}
consistent with the Otto--Benamou--Brenier geometry \citep{BenamouBrenier2000,Otto2001,AmbrosioGigliSavare2008}.

\subsection{Proof of the local defect formula}
Let $q_{t,h}^w=(\operatorname{Id}+hw_t)_\#p_t$.  For a smooth test function $\varphi$,
\begin{align}
\int\varphi\,dq_{t,h}^w
&=\int\varphi(x+hw_t(x))p_t(x)\,dx\\
&=\int\varphi p_t
+h\int\langle\nabla\varphi,w_t\rangle p_t
+\frac{h^2}{2}\int w_t^\top\nabla^2\varphi\,w_t\,p_t
+O(h^3).
\label{eq:euler-pushforward-test}
\end{align}
For $F(t)=\int\varphi p_t$, the continuity equation and the transport differentiation rule give
\begin{align}
F'(t)&=\int\langle\nabla\varphi,w_t\rangle p_t,\\
F''(t)&=\int\left(w_t^\top\nabla^2\varphi\,w_t+\langle\nabla\varphi,a_w\rangle\right)p_t.
\end{align}
Consequently,
\begin{equation}
\int\varphi\,(q_{t,h}^w-p_{t+h})
=-\frac{h^2}{2}\int\langle\nabla\varphi,a_w\rangle p_t+O(h^3),
\end{equation}
or, in distributions,
\begin{equation}
q_{t,h}^w-p_{t+h}
=\frac{h^2}{2}\nabla\cdot(p_ta_w)+O(h^3)
=-\frac{h^2}{2}\nabla\cdot(p_t(-a_w))+O(h^3).
\label{eq:euler-exact-density-difference}
\end{equation}
Applying \eqref{eq:w2-tangent-chart} with $\varepsilon=h^2/2$ and $\xi=-a_w$ yields
\begin{equation}
W_2(q_{t,h}^w,p_{t+h})
=\frac{h^2}{2}\|P_{\nabla,p_t}a_w\|_{L^2(p_t)}+o(h^2).
\end{equation}
Proposition~\ref{prop:hminus-main} gives the equivalent $\dot H^{-1}(p_t)$ formula.

Squaring the expansion gives
\begin{equation}
\frac{2}{h_k^3}W_2^2(q_{t_k,h_k}^w,p_{t_{k+1}})
=\frac{h_k}{2}\|P_{\nabla,p_{t_k}}a_w\|_{L^2(p_{t_k})}^2+o(h_k).
\end{equation}
Thus, for regular refining partitions $\pi=\{t_k\}$,
\begin{equation}
\cB(w)=\lim_{|\pi|\to0}\sum_k\frac{2}{h_k^3}
W_2^2\!\left((\operatorname{Id}+h_kw_{t_k})_\#p_{t_k},p_{t_{k+1}}\right).
\label{eq:bias-as-w2-action}
\end{equation}  Likewise,
\begin{equation}
\sum_kW_2(q_{t_k,h_k}^w,p_{t_{k+1}})
=\frac12\sum_k h_k^2R_{W_2}(t_k)+o\!\left(\sum_kh_k^2\right).
\end{equation}
Writing $h(t)=n(t)^{-1}$ converts the leading Riemann sum to
$\frac12\int R_{W_2}(t)n(t)^{-1}dt$, which is the Euler case of the
solver-order allocation rule \eqref{eq:order-r-allocation}.

\subsection{Isotropic Gaussian separation}
Let $p_t=\mathcal N(0,s(t)^2I_d)$.  For the gradient lift $v_t(x)=\beta(t)x$ with $\beta=\dot s/s$, Euler produces
\begin{equation}
(\operatorname{Id}+hv_t)_\#p_t
=\mathcal N\!\left(0,(s+h\dot s)^2I_d\right).
\end{equation}
For centered isotropic Gaussians, $W_2(\mathcal N(0,\sigma_1^2I_d),\mathcal N(0,\sigma_2^2I_d))=\sqrt d\,|\sigma_1-\sigma_2|$.  Since
\begin{equation}
s(t+h)=s+h\dot s+\frac{h^2}{2}\ddot s+O(h^3),
\end{equation}
we obtain the displayed $O(h^2)$ formula in Section~\ref{sec:gauge-consequences-main}.

Now let $B=\beta I_d+A$ with $A^\top=-A$ and $A^2=-\kappa I_d$, $\kappa=\ddot s/s>0$.  The Euler covariance is
\begin{align}
\Sigma^E_{t+h}
&=s^2(I+hB)(I+hB)^\top\\
&=s^2\left[I+2h\beta I+h^2(\beta^2I-A^2)\right]\\
&=s^2\left[I+2h\beta I+h^2(\beta^2+\kappa)I\right].
\end{align}
On the other hand,
\begin{equation}
s(t+h)^2
=s^2+2hs\dot s+h^2(\dot s^2+s\ddot s)+O(h^3)
=s^2\left[1+2h\beta+h^2(\beta^2+\kappa)\right]+O(h^3).
\end{equation}
The covariance, hence its positive square root, agrees through order $h^2$, which proves the displayed $O(h^3)$ formula in Section~\ref{sec:gauge-consequences-main}.  The field $Ax$ is nonzero and rotational, so the lift is not gradient and is not the Otto-horizontal minimum-kinetic-energy representative.  It is nevertheless Euler-optimal to the displayed order.

%% file: appendices/O_solver_aware_variational_limits.tex
\section{Solver-Aware Extensions and Variational Limits}
\label{app:solver-aware-limits}

\begin{takeawaybox}
\textbf{Variational takeaway.} This appendix collects the proof details linking local modified equations, propagated allocation, robustness to residual misspecification, and the Euler $\Gamma$-limit.  The assumptions are stated uniformly over admissible representatives so that the resulting objectives can be used for optimization, not only pointwise error expansion.
\end{takeawaybox}

This appendix proves the order-$r$ local result, the exact allocation and
robustness statements, the propagation separation, and the Euler
$\Gamma$-convergence theorem.  The assumptions are stated so that the
uniformity required for optimization is explicit.

\subsection{\texorpdfstring{Order-$r$}{Order-r} density and Wasserstein defects}
Assume that the numerical map $\Psi_h^w$ is the time-$h$ map of
\[
w_h=w+h^r\mathcal E_r[w]+h^{r+1}\rho_h[w],
\qquad
\sup_{w\in\mathcal A,h<h_0}\|\rho_h[w]\|_{C_x^1}<\infty .
\]
Subtracting the continuity equation for $w$ gives
\[
\partial_tp_t+\nabla\cdot(p_tw_h)
=h^r\nabla\cdot(p_t\mathcal E_r[w])
+h^{r+1}\nabla\cdot(p_t\rho_h[w]),
\]
which proves \eqref{eq:residual-general}.  Equivalently, linearize the
continuity equation around $w$ in the velocity perturbation
$h^r\mathcal E_r[w]+O(h^{r+1})$.  Over a time interval of length $h$, variation
of constants yields, for every smooth test function $\varphi$,
\begin{align}
\int\varphi\,d\!\left((\Psi_h^w)_\#p_t-p_{t+h}\right)
&=h^{r+1}\int p_t
\langle\mathcal E_r[w],\nabla\varphi\rangle\,dx
+o(h^{r+1})\\
&=\left\langle
-h^{r+1}\nabla\cdot(p_t\mathcal E_r[w]),\varphi
\right\rangle+o(h^{r+1}).
\end{align}
This proves \eqref{eq:one-step-density-general}.

In a regular Wasserstein chart, a density perturbation
$-\varepsilon\nabla\cdot(p_tz)+o(\varepsilon)$ has tangent norm
\[
W_2\!\left(p_t-\varepsilon\nabla\cdot(p_tz)+o(\varepsilon),p_t\right)
=|\varepsilon|\|P_{\nabla,p_t}z\|_{L^2(p_t)}+o(|\varepsilon|).
\]
Taking $\varepsilon=h^{r+1}$ and $z=\mathcal E_r[w]$ proves
Theorem~\ref{thm:local-w2-defect}.  Uniformity over $\mathcal A$ follows if
the modified-equation remainder and the Wasserstein-chart remainder are
uniform there.  For Euler, $\mathcal E_1=-a_w/2$, so the formula agrees with
the direct calculation in Appendix~\ref{app:local-ot}.

\subsection{Propagation and solver order}
The velocity-to-density derivative satisfies
\[
\eta_1=-\int_0^1U_w(1,t)\nabla\cdot(p_tb_t)\,dt
\]
as shown in Appendix~\ref{app:propagation-proof}.  A nonuniform order-$r$
modified field has leading perturbation
$b_t=h(t)^r\mathcal E_r[w](t)$ and a remainder
$O(\|h\|_\infty^{r+1})$, provided the step profile varies regularly.
Applying \eqref{eq:transport-stability} gives
\[
\|\eta_1\|_X
\le\int_0^1K(1,t)h(t)^r
\|\nabla\cdot(p_t\mathcal E_r[w])\|_Y\,dt,
\]
and the differentiability remainder proves
Theorem~\ref{thm:propagated-defect}.

\subsection{Exact continuous and discrete allocation}
For generic $q>0$, write
\[
\mathcal J_q^W[n]=\int_0^1W(t)n(t)^{-q}\,dt,
\qquad
\int_0^1n=N.
\]
The Lagrangian derivative is
$-qWn^{-q-1}+\lambda$.  Strict convexity gives the unique stationary point
\[
n^\star(t)=\frac{NW(t)^{1/(q+1)}}{\int_0^1W^{1/(q+1)}}.
\]

For the discrete statement, set
$\alpha=1/(q+1)$,
$a_k=\int_{t_k}^{t_{k+1}}W^\alpha$, and
$Z=\sum_ka_k$.  Jensen's inequality for the strictly convex map
$x\mapsto x^{q+1}$ gives
\[
\mathcal S_q(\pi)=\sum_ka_k^{q+1}
\ge N\left(\frac ZN\right)^{q+1}
=\frac{Z^{q+1}}{N^q}.
\]
Equality holds exactly when every $a_k=Z/N$, which is
\eqref{eq:exact-quantile-grid}.

For the quantile-mesh asymptotics, the mean-value theorem gives
\[
\frac ZN
=\int_{t_k}^{t_{k+1}}W(t)^\alpha\,dt
=W(\xi_k)^\alpha h_k
\]
for some $\xi_k$ in the interval.  Since $W\ge m>0$, the quantile density is
bounded below, so $h_k=O(N^{-1})$.  If $W\in C^1$, then
$W(\xi_k)^{-\alpha}=W(t_k)^{-\alpha}+O(h_k)$ uniformly; multiplication by
$N^{-1}$ yields the $O(N^{-2})$ endpoint approximation.

\subsection{R\'enyi robustness identity}
Let $\alpha=1/(q+1)$ and use \eqref{eq:schedule-densities}.  Since
$W=Z^{q+1}\rho^{q+1}$,
\begin{align}
\mathcal J_q^W[N\widehat\rho]
&=N^{-q}Z^{q+1}
\int_0^1\rho^{q+1}\widehat\rho^{-q}\,dt,\\
\mathcal J_q^W[N\rho]&=N^{-q}Z^{q+1}.
\end{align}
By the definition
\[
D_\beta(\rho\Vert\widehat\rho)
=\frac1{\beta-1}\log\int\rho^\beta
\widehat\rho^{1-\beta}\,dt,
\]
the ratio is $\exp(qD_{q+1})$.  When $q=1$,
\[
\int\frac{\rho^2}{\widehat\rho}
=1+\int\frac{(\rho-\widehat\rho)^2}{\widehat\rho},
\]
which proves \eqref{eq:chi2-robustness}.

If $|\log(\widehat W/W)|\le\varepsilon$, then
$e^{-\alpha\varepsilon}\le\widehat W^\alpha/W^\alpha
\le e^{\alpha\varepsilon}$ and the same bounds hold for $\widehat Z/Z$.
Thus
$e^{-2\alpha\varepsilon}\le\widehat\rho/\rho
\le e^{2\alpha\varepsilon}$, and
\[
\int\rho^{q+1}\widehat\rho^{-q}
=\int\rho\left(\frac{\rho}{\widehat\rho}\right)^q
\le e^{2q\varepsilon/(q+1)}.
\]

\subsection{Propagation separation}
Let $R\equiv1$ and let $K_M=M$ on a set $A_M$ of measure
$\delta_M=M^{-1/2}$ and $K_M=1$ elsewhere.  Residual-only scheduling is
uniform, $n_R=N$, and
\[
\mathcal J_1^{K_M}[n_R]
=\frac1N(1-\delta_M+M\delta_M).
\]
The optimal value is
\[
\mathcal J_1^{K_M}[n^\star]
=\frac1N\left(1-\delta_M+\sqrt M\,\delta_M\right)^2
=\frac1N(2-M^{-1/2})^2.
\]
Consequently,
\[
\frac{\mathcal J_1^{K_M}[n_R]}
{\mathcal J_1^{K_M}[n^\star]}
=\frac{1-M^{-1/2}+M^{1/2}}{(2-M^{-1/2})^2}
\sim\frac{\sqrt M}{4}.
\]
Smooth positive approximations of $K_M$ preserve the divergence of the ratio.
Such profiles are realizable as scalar linear amplification: for
$\dot x=\lambda(t)x$,
$K(1,t)=\exp(\int_t^1\lambda(s)\,ds)$, so a smooth target profile with
$K(1,1)=1$ is obtained from
$\lambda(t)=-\partial_t\log K(1,t)$.

\subsection{\texorpdfstring{$\Gamma$-convergence}{Gamma-convergence} of Euler gauge objectives}
\label{app:gamma-proof}
For $w\in\mathcal A$, write
$R_w(t)=\|P_{\nabla,p_t}a[w](t)\|_{L^2(p_t)}$.
The assumed uniform local expansion is
\[
W_2\!\left(
(\operatorname{Id}+hw_t)_\#p_t,p_{t+h}
\right)
=\frac{h^2}{2}R_w(t)+h^3e_h(w,t),
\qquad
\sup|e_h|<\infty .
\]
Compactness and continuity give $\sup_{\mathcal A,t}R_w(t)<\infty$.
After squaring,
\[
\frac{2}{h_k^3}W_2^2
=\frac{h_k}{2}R_w(t_k)^2+O(h_k^2)
\]
uniformly in $w$.  Since
$\sum_kh_k^2\le|\pi|\sum_kh_k=|\pi|$ and the squared residuals are uniformly
Lipschitz in time,
\[
\sup_{w\in\mathcal A}
\left|
\mathcal B_\pi(w)
-\frac12\int_0^1R_w(t)^2\,dt
\right|
\le C|\pi|.
\]
This proves \eqref{eq:uniform-gauge-convergence}.

Uniform convergence implies the $\Gamma$-liminf inequality: if
$w_\pi\to w$ in $\tau$, continuity of $\mathcal B$ and the uniform error give
$\liminf\mathcal B_\pi(w_\pi)\ge\mathcal B(w)$.  The constant recovery
sequence $w_\pi=w$ gives the matching $\Gamma$-limsup.  Compactness yields
convergent subsequences of minimizers, and the fundamental minimizer argument
shows that every limit minimizes $\mathcal B$.

On a noncompact admissible set, the same conclusion holds for
\[
\mathcal B_\pi(v+u)+\frac\lambda2\|u\|_{L^2(p_tdt)}^2
\]
whenever the regularized family is equicoercive and the local expansion is
uniform on its sublevel sets.  These hypotheses are essential: pointwise
local asymptotics alone do not imply convergence of minimizing sequences.

%% file: appendices/A_projected_gauge_optimization.tex
\section{Projected Gauge Optimization}
\label{app:projected-gauge}

We now derive a projected optimization step for the solenoidal-constrained representative-selection problem
\begin{equation}
    \min_{u\in\cU_p}\; \cB(v+u)+\frac{\lambda}{2}\int_0^1\mathbb E_{p_t}\|u_t\|^2\,dt.
    \label{eq:main-opt}
\end{equation}
Fix a feasible lift $v_t$ and consider a vertical perturbation $v^\alpha=v+\alpha u$, where $u\in\cU_p$.  Assume $u$ vanishes at the temporal endpoints, or equivalently restrict to compactly supported variations in time.

The acceleration expands as
\begin{equation}
    a_{v+\alpha u}
    =a_v+\alpha\mathcal L_vu+\alpha^2\nabla u\,u,
    \label{eq:acc-exp}
\end{equation}
with
\begin{equation}
    \mathcal L_vu=\partial_tu+\nabla u\,v+\nabla v\,u.
\end{equation}
Let
\begin{equation}
    b_v=P_{\nabla,p_t}a_v.
\end{equation}
Because $P_{\nabla,p_t}$ is self-adjoint and $b_v$ lies in its range,
\begin{equation}
    D\cB(v)[u]
    =\int_0^1\langle b_v,\mathcal L_vu\rangle_{p_t}\,\dd t.
    \label{eq:first-variation}
\end{equation}

Define the material derivative $D_t^v f=\partial_t f+(v\cdot\nabla)f$.  The continuity equation implies the transport integration-by-parts identity
\begin{equation}
    \int_0^1\langle b,D_t^v u\rangle_{p_t}\,dt
    =-\int_0^1\langle D_t^v b,u\rangle_{p_t}\,dt,
    \label{eq:transport-ibp}
\end{equation}
whenever the temporal boundary term $\langle b,u\rangle_{p_t}|_{0}^{1}$ vanishes.  Applying this to \eqref{eq:first-variation} gives
\begin{equation}
    D\cB(v)[u]
    =\int_0^1\langle z_v,u\rangle_{p_t}\,\dd t,
    \qquad
    z_v=-D_t^v b_v+(\nabla v)^\top b_v .
    \label{eq:zv}
\end{equation}

\begin{theorem}[Projected descent and stationarity]
\label{thm:projected-main}
Let $P_{\cU_p}$ be the orthogonal projection onto $\cU_p$ in $L^2(p_t\dd t)$.  The direction
\begin{equation}
    u^\star=-P_{\cU_p}z_v
    \label{eq:star-direction}
\end{equation}
solves the steepest linearized descent problem over unit-norm $p_t$-solenoidal directions.  Moreover,
\begin{equation}
    D\cB(v)[u^\star]
    =-\|P_{\cU_p}z_v\|_{L^2(p_t\dd t)}^2\le0.
    \label{eq:descent}
\end{equation}
Hence, if $P_{\cU_p}z_v\ne0$, then $\cB(v+\alpha u^\star)<\cB(v)$ for all sufficiently small $\alpha>0$.  A feasible lift is first-order stationary for \eqref{eq:bias-objective} inside its gauge class iff
\begin{equation}
    P_{\cU_p}z_v=0.
    \label{eq:stationarity}
\end{equation}
\end{theorem}

\begin{proof}
For feasible directions $u\in\cU_p$, \eqref{eq:zv} gives
$D\cB(v)[u]=\langle P_{\cU_p}z_v,u\rangle_{L^2(p_tdt)}$, because $u$ is orthogonal to the gradient component of $z_v$.  Minimizing this linear functional over $\|u\|_{L^2(p_tdt)}\le1$ gives the normalized negative projected gradient; without normalization the steepest descent vector is $-P_{\cU_p}z_v$.  Substitution yields \eqref{eq:descent}.  If the projected gradient is nonzero, differentiability of $\cB$ gives strict decrease for sufficiently small positive steps.  The stationarity condition is exactly vanishing of the derivative against all feasible directions, equivalently $P_{\cU_p}z_v=0$.
\end{proof}

The projection onto $\cU_p$ is again a weighted Poisson solve:
\begin{equation}
    P_{\cU_p}z=z-\nabla\phi,
    \qquad
    \nabla\cdot(p_t\nabla\phi_t)=\nabla\cdot(p_tz_t).
    \label{eq:poisson-sol}
\end{equation}
Thus each projected step consists of evaluating the bias gradient and solving a scalar elliptic projection problem at each time.

\begin{proposition}[Weighted Poisson projection]
\label{prop:poisson-projection}
For any $z_t\in L^2(p_t)$, let $\phi_t$ be the zero-mean weak solution of
\begin{equation}
    \int p_t\langle\nabla\phi_t,\nabla\psi\rangle\,dx
    =\int p_t\langle z_t,\nabla\psi\rangle\,dx
    \qquad \forall \psi .
\end{equation}
Then $P_{\nabla,p_t}z_t=\nabla\phi_t$ and
\begin{equation}
    P_{\cU_p}z_t=z_t-\nabla\phi_t.
\end{equation}
In strong form, $\phi_t$ solves $\nabla\cdot(p_t\nabla\phi_t)=\nabla\cdot(p_tz_t)$ with the usual normalization modulo constants.
\end{proposition}

\begin{proof}
The weak problem is the normal equation for the orthogonal projection of $z_t$ onto the closed gradient subspace $\mathcal G_{p_t}$.  Hence $z_t-\nabla\phi_t$ is orthogonal to all gradients, which is precisely the weak definition of $\cU_{p_t}$ in \eqref{eq:weak-solenoidal}.  The strong form follows by integration by parts when $z_t,p_t$ are smooth.
\end{proof}

\begin{corollary}[Closed-form regularized subproblem]
The linearized regularized problem
\begin{equation}
    \min_{u\in\cU_p}\; \langle z_v,u\rangle_{p}+\frac{\lambda}{2}\|u\|_{p}^2
\end{equation}
has the unique solution
\begin{equation}
    u_\lambda^\star=-\frac{1}{\lambda}P_{\cU_p}z_v.
\end{equation}
\end{corollary}

\begin{theorem}[KKT system for regularized representative selection]
\label{thm:kkt-main}
Let $u$ be a smooth local minimizer of
\[
\min_{u\in\cU_p}\;\cB(v+u)+\frac\lambda2\|u\|_{L^2(p_tdt)}^2 .
\]
Set $w=v+u$, $b_w=P_{\nabla,p}a_w$, and $z_w=-D_t^w b_w+(\nabla w)^\top b_w$.  Then
\[
P_{\cU_p}(z_w+\lambda u)=0,
\qquad
\nabla\cdot(p_tu_t)=0.
\]
Equivalently, under the regularity assumptions ensuring that gradient elements admit potential representatives, there exists a scalar multiplier $\pi_t$ such that
\[
z_w+\lambda u=\nabla\pi_t,
\qquad
\nabla\cdot(p_tu_t)=0.
\]
\end{theorem}

\begin{proof}
Taking variations $\delta u\in\cU_p$ gives $\langle z_w+\lambda u,\delta u\rangle_p=0$ for all feasible variations.  Thus $z_w+\lambda u$ lies in the weighted-orthogonal complement of $\cU_p$, which is the closed gradient subspace.  The multiplier form follows whenever this gradient element admits a potential representative; otherwise the equality is understood in the closure of gradients.
\end{proof}

\begin{theorem}[Linearized exact cancellation and minimal residual]
\label{thm:lin-cancel-main}
Let $\mathfrak R_v:\cU_p\to L^2(p_t)$ be the linear map
\begin{equation}
    \mathfrak R_vu=P_{\nabla,p_t}\bigl(\mathcal L_vu\bigr).
\end{equation}
The linearized problem of cancelling the leading Euler density bias,
\begin{equation}
    P_{\nabla,p_t}\bigl(a_v+\mathcal L_vu\bigr)=0,
    \qquad u\in\cU_p,
    \label{eq:lin-cancel}
\end{equation}
is exactly solvable if and only if
\begin{equation}
    -b_v\in \mathfrak R_v(\cU_p).
\end{equation}
Moreover, the minimal linearized residual is
\begin{equation}
    \inf_{u\in\cU_p}
    \|b_v+\mathfrak R_vu\|_{L^2(p_tdt)}^2
    =
    \operatorname{dist}^2\bigl(-b_v,\overline{\mathfrak R_v(\cU_p)}\bigr),
    \label{eq:min-residual}
\end{equation}
and the infimum is zero if and only if $-b_v\in\overline{\mathfrak R_v(\cU_p)}$.
\end{theorem}

\begin{proof}
The first-order expansion \eqref{eq:acc-exp} gives
$P_{\nabla,p}(a_{v+u})=b_v+\mathfrak R_vu+O(\|u\|^2)$.
Thus exact first-order cancellation is equivalent to solving $\mathfrak R_vu=-b_v$ over the feasible linear space $\cU_p$.  If the range is not closed, exact solvability requires membership in the actual range, while arbitrarily small residual requires membership in the range closure.  Formula \eqref{eq:min-residual} is the Hilbert-space projection identity for the distance to the attainable affine subspace.
\end{proof}

\begin{theorem}[Projected-gradient decrease after discretization]
Consider a finite-dimensional parameterization of feasible gauges $u_\theta\in\cU_p$ or a spatial-temporal discretization of \eqref{eq:main-opt}.  Let $F(\theta)=\cB(v+u_\theta)$ and assume $F$ has an $L$-Lipschitz gradient on a neighborhood of $\theta$.  Then the projected-gradient step
\begin{equation}
    \theta^+=\theta-\alpha \nabla F(\theta),
    \qquad 0<\alpha\le L^{-1},
\end{equation}
satisfies
\begin{equation}
    F(\theta^+)\le F(\theta)-\frac{\alpha}{2}\|\nabla F(\theta)\|^2.
\end{equation}
In the Hilbert-space formalism this corresponds to the infinitesimal direction
$-P_{\cU_p}z_v$ from \eqref{eq:star-direction}.
\end{theorem}

\begin{proof}
The result is the standard descent lemma for smooth unconstrained optimization after the solenoidal constraint has been built into the parameterization or after applying the orthogonal projection.  The identification of the projected gradient follows from \eqref{eq:zv}.
\end{proof}

\begin{algorithm}[H]
\caption{Projected modified-equation representative-selection step}
\label{alg:gauge}
\begin{algorithmic}[1]
\STATE Input feasible velocity $v_t$, density path $p_t$, step size $\alpha$.
\STATE Compute $a_v=\partial_t v+\nabla v\,v$.
\STATE Compute $b_v=P_{\nabla,p_t}a_v$ by solving \eqref{eq:poisson-grad}.
\STATE Compute $z_v=-D_t^v b_v+(\nabla v)^\top b_v$.
\STATE Compute $P_{\cU_p}z_v=z_v-\nabla\phi$ by solving \eqref{eq:poisson-sol}.
\STATE Return $v_t^{+}=v_t-\alpha P_{\cU_p}z_v$.
\end{algorithmic}
\end{algorithm}

%% file: appendices/B_exact_and_obstructed_gaussian_gauges.tex
\section{Exact and Obstructed Gaussian Gauges}
\label{app:gaussian-gauge}

\begin{takeawaybox}
\textbf{Exact Gaussian takeaway.} This appendix gives the explicit finite-dimensional model behind the main Gaussian claim.  The Lyapunov equation fixes the continuum density path, while the Riccati-type equation identifies when a linear gauge cancels the leading Euler covariance defect.  The isotropic case then reduces to a parity obstruction for skew-symmetric controls.
\end{takeawaybox}

We now show that the proposed optimization problem has nontrivial explicit solutions.  Let $p_t=\mathcal N(0,\Sigma_t)$ with $\Sigma_t\succ0$ and consider linear lifts
\begin{equation}
    w_t(x)=B_tx.
\end{equation}
The covariance evolves according to
\begin{equation}
    \dot\Sigma_t=B_t\Sigma_t+\Sigma_tB_t^\top.
    \label{eq:lyapunov}
\end{equation}
Thus the gauge class consists of all matrices $B_t$ solving the Lyapunov constraint \eqref{eq:lyapunov}.  The material acceleration is linear,
\begin{equation}
    a_w(x)=M_tx,
    \qquad
    M_t=\dot B_t+B_t^2.
\end{equation}
The acceleration is density-invisible iff
\begin{equation}
    M_t\Sigma_t+\Sigma_tM_t^\top=0.
    \label{eq:invisible-gauss}
\end{equation}
Combining \eqref{eq:lyapunov} with time differentiation gives the following finite-dimensional characterization.

\begin{theorem}[Gaussian Lyapunov--Riccati gauge]
\label{thm:gaussian-main}
A linear Gaussian lift $w_t(x)=B_tx$ transports $\mathcal N(0,\Sigma_t)$ and has zero leading Euler covariance bias within the linear Gaussian class iff
\begin{equation}
    B_t\Sigma_t+\Sigma_tB_t^\top=\dot\Sigma_t,
    \qquad
    2B_t\Sigma_tB_t^\top=\ddot\Sigma_t.
    \label{eq:gaussian-system}
\end{equation}
Equivalently, it transports the path and satisfies $\nabla\cdot(p_ta_w)=0$ for the linear acceleration $a_w(x)=(\dot B_t+B_t^2)x$.
\end{theorem}

\begin{proof}
The first equation is the covariance continuity condition.  One Euler step gives
\[
\Sigma_{t+h}^{E}=(I+hB_t)\Sigma_t(I+hB_t)^\top
=\Sigma_t+h\dot\Sigma_t+h^2B_t\Sigma_tB_t^\top.
\]
The exact covariance path satisfies
\[
\Sigma_{t+h}=\Sigma_t+h\dot\Sigma_t+\frac{h^2}{2}\ddot\Sigma_t+O(h^3).
\]
The order-$h^2$ covariance bias vanishes iff the second equation in \eqref{eq:gaussian-system} holds.  Differentiating \eqref{eq:lyapunov} gives
\[
\ddot\Sigma_t=(\dot B_t+B_t^2)\Sigma_t+\Sigma_t(\dot B_t+B_t^2)^\top+2B_t\Sigma_tB_t^\top.
\]
Thus the second equation in \eqref{eq:gaussian-system} is equivalent to $M_t\Sigma_t+\Sigma_tM_t^\top=0$, i.e. \eqref{eq:invisible-gauss}.  For a zero-mean Gaussian and linear acceleration $M_tx$, this is exactly the condition $\nabla\cdot(p_tM_tx)=0$.
\end{proof}

\begin{corollary}[Covariance cone obstruction]
\label{cor:cov-cone}
For a Gaussian covariance path, exact one-step second-order Euler covariance matching by a linear lift requires
\[
\ddot\Sigma_t\succeq0.
\]
Indeed, the exact matching condition in Theorem~\ref{thm:gaussian-main} reads $\ddot\Sigma_t=2B_t\Sigma_tB_t^\top$, and the right-hand side is positive semidefinite.  Thus some covariance paths cannot be made second-order accurate by any linear marginal-preserving gauge.
\end{corollary}

\subsection{Isotropic classification and parity obstruction}

Let $\Sigma_t=s(t)^2I_d$.  Then every feasible linear lift has the form
\begin{equation}
    B_t=\beta(t)I_d+A_t,
    \qquad
    \beta(t)=\frac{\dot s(t)}{s(t)},
    \qquad
    A_t^\top=-A_t.
\end{equation}
The exact second-order condition reduces to
\begin{equation}
    A_t^2=-\kappa(t)I_d,
    \qquad
    \kappa(t)=\frac{\ddot s(t)}{s(t)}.
    \label{eq:A-condition}
\end{equation}

\begin{theorem}[Parity obstruction]
\label{thm:parity}
Fix $t$ and set $\kappa=\ddot s/s$.  Then
\begin{equation}
\min_{A^\top=-A}\|A^2+\kappa I_d\|_F^2
=
\begin{cases}
0, & \kappa>0 \text{ and } d \text{ even},\\
\kappa^2, & \kappa>0 \text{ and } d \text{ odd},\\
d\kappa^2, & \kappa\le0.
\end{cases}
\label{eq:parity}
\end{equation}
When $d=2m$ and $\kappa>0$, the exact gauges are
\begin{equation}
    A=\sqrt{\kappa}\,J,
    \qquad
    J^\top=-J,
    \qquad
    J^2=-I_d.
\end{equation}
Their set is isomorphic to the homogeneous space $O(2m)/U(m)$.
\end{theorem}

\begin{proof}
The eigenvalues of a real skew-symmetric matrix are pairs $\pm i\omega_j$ and, if $d$ is odd, one zero eigenvalue.  Hence the eigenvalues of $A^2$ are $-\omega_j^2$ with even multiplicity, plus possibly one zero.  The minimization of $\|A^2+\kappa I\|_F^2$ therefore separates over these blocks.  If $\kappa>0$, each two-dimensional block is minimized by $\omega_j^2=\kappa$.  In odd dimension one zero direction remains, giving residual $\kappa^2$.  If $\kappa\le0$, choosing $\omega_j\ne0$ only moves the eigenvalues further away from $-\kappa\ge0$, so $A=0$ is optimal.  The last claim is the standard parametrization of orthogonal complex structures.
\end{proof}

The theorem gives a sharp no-free-lunch result: even in the simplest Gaussian family, exact solver-bias cancellation can be impossible because of dimension and sign constraints.  Conversely, in even dimensions with $\ddot s/s>0$, the continuum path admits an entire manifold of exact discretization-cancelling gauges.

\begin{takeawaybox}
\textbf{Geometric meaning.}  In even dimension, exact isotropic gauges are not isolated.  They are parametrized by orthogonal complex structures $J$, hence by $O(2m)/U(m)$.  The rotational gauge is therefore not an arbitrary trick: it is the finite-dimensional shadow of choosing a numerically favorable horizontal representative inside a gauge fiber.
\end{takeawaybox}

\begin{proposition}[Finite-dimensional Gaussian gauge optimization]
For a fixed triple $(\Sigma,\dot\Sigma,\ddot\Sigma)$, the linear Gaussian representative-selection problem can be written as
\[
\min_{B}\left\|B\Sigma B^\top-\frac12\ddot\Sigma\right\|_F^2
\quad\text{s.t.}\quad
B\Sigma+\Sigma B^\top=\dot\Sigma .
\]
The objective is the order-$h^2$ Euler covariance bias.  Zero objective is attainable exactly when the Lyapunov--Riccati system \eqref{eq:gaussian-system} is solvable.
\end{proposition}

%% file: appendices/K_additional_projection_details.tex
\section{Additional Projection Details}
\label{app:proofs}

For completeness, we recall the projection argument.  A vector field $u$ is in $\cU_p$ iff $\nabla\cdot(pu)=0$.  For any smooth potential $\phi$,
\begin{equation}
    \int p\langle u,\nabla\phi\rangle\,\dd x
    =-\int \phi\nabla\cdot(pu)\,\dd x=0.
\end{equation}
Thus gradients are orthogonal to $\cU_p$.  To project $z$ onto $\cU_p$, write $P_{\cU_p}z=z-\nabla\phi$ and impose
\begin{equation}
    0=\nabla\cdot(p(z-\nabla\phi)),
\end{equation}
which yields
\begin{equation}
    \nabla\cdot(p\nabla\phi)=\nabla\cdot(pz).
\end{equation}

In $d$ dimensions, a sufficient way to enforce $\nabla\cdot(pu)=0$ is to introduce a skew-symmetric matrix potential $A_{ij}=-A_{ji}$ and set
\begin{equation}
    pu=m,
    \qquad
    m_i=\sum_{j=1}^d\partial_j A_{ij}.
\end{equation}
Then
\begin{equation}
    \nabla\cdot m=\sum_{i,j}\partial_i\partial_jA_{ij}=0
\end{equation}
by symmetry of second derivatives and skew-symmetry of $A$.  This is the higher-dimensional analogue of the stream-function representation in two dimensions.

%% file: appendices/L_propagated_error_and_cost_accounting.tex
\section{Propagation of the Density Defect and Cost Accounting}
\label{app:propagation-proof}

\subsection{Proof of Theorem~\ref{thm:propagated-defect}}
Let
\[
\mathcal F(p,w)=\partial_t p+\nabla\cdot(pw).
\]
By assumption, the solution map $w\mapsto p[w]$ is differentiable from a neighborhood of $v$ into the chosen trajectory space with terminal norm $X$.  Differentiating
\[
\mathcal F(p^\varepsilon,v+\varepsilon b)=0
\]
at $\varepsilon=0$ gives
\begin{equation}
\partial_t\eta_t+\nabla\cdot(v_t\eta_t)
=-\nabla\cdot(p_tb_t),
\qquad \eta_0=0.
\label{eq:linearized-density-app}
\end{equation}
Variation of constants for the evolution family $U_v(t,s)$ associated with the homogeneous linear transport equation yields
\begin{equation}
\eta_t=-\int_0^tU_v(t,s)\nabla\cdot(p_sb_s)\,ds.
\label{eq:duhamel-app}
\end{equation}
Setting $t=1$ gives the terminal first-variation formula used in
Theorem~\ref{thm:propagated-defect}.  The stability assumption
\eqref{eq:transport-stability}, the triangle inequality, and Bochner
integrability give
\[
\|\eta_1\|_X
\le \int_0^1
\|U_v(1,t)\nabla\cdot(p_tb_t)\|_X\,dt
\le \int_0^1K(1,t)\|\nabla\cdot(p_tb_t)\|_Y\,dt,
\]
which is the stated stability estimate.

For a smoothly varying nonuniform Euler grid, let $h(t)$ be a piecewise-smooth interpolation of the local step lengths.  Standard modified-equation interpolation gives, locally on each interval,
\begin{equation}
v_h(t,x)=v(t,x)-\frac{h(t)}2a_v(t,x)+r_h(t,x),
\qquad
\|r_h\|=O(\|h\|_\infty^2),
\label{eq:nonuniform-modified-field-app}
\end{equation}
under the same regularity assumptions as Theorem~\ref{thm:density-bea}, with additional control of the variation of $h$.  Applying the first variation with
$b_t=-\frac12h(t)a_{v_t}$ gives the leading term
\[
\frac12\int_0^1K(1,t)h(t)
\|\nabla\cdot(p_ta_{v_t})\|_Y\,dt.
\]
If the velocity-to-terminal-density map is twice Fr\'echet differentiable and the remainder in \eqref{eq:nonuniform-modified-field-app} is uniformly controlled, the omitted terms are $O(\|h\|_\infty^2)$ in $X$.  This proves the Euler specialization of \eqref{eq:order-r-final-bound}.

\begin{remark}[Choice of terminal and source norms]
The theorem is deliberately norm-agnostic.  A natural source norm is $Y=\dot H^{-1}(p_t)$ because Proposition~\ref{prop:hminus-main} identifies it with the weighted-gradient component of the acceleration.  The terminal norm $X$ may be another negative Sobolev norm, a dual-Lipschitz norm, or a problem-specific weak metric, provided the transport evolution admits \eqref{eq:transport-stability}.  The propagation factor $K(1,t)$ depends on this choice.
\end{remark}

\begin{remark}[Residual-only schedule]
Setting $K(1,t)=1$ gives the local marginal-defect schedule and is exact when the chosen transport norm is nonexpansive along the reference flow.  Directional amplification in Appendix~\ref{app:unified-new-experiments} provides a propagation-aware extension.
\end{remark}

\subsection{Offline profile cost and amortization}
Let $M$ be the number of profile times, $S$ the number of profile particles, and $H$ the number of trace probes.  A schematic model-equivalent profile cost is
\begin{equation}
C_{\rm prof}
\approx MS\bigl(c_{\rm acc}+Hc_{\rm tr}\bigr),
\label{eq:profile-cost}
\end{equation}
where $c_{\rm acc}$ accounts for evaluating or differentiating the material acceleration and $c_{\rm tr}$ accounts for one divergence probe.  The constants depend on whether automatic differentiation, finite differences, JVPs, or VJPs are used, so particle--probe counts should not be conflated with exact wall-clock NFE.

If one offline profile is reused for $B$ generated samples, then a sampler with $N$ field evaluations per sample has amortized model-equivalent cost
\begin{equation}
C_{\rm amortized}(B)=N+\frac{C_{\rm prof}}{B}.
\label{eq:amortized-cost}
\end{equation}
Consequently, a schedule can be attractive for large batches even when profiling is expensive.  Appendix~\ref{app:unified-new-experiments} reports implementation-specific CPU wall-clock measurements and the resulting break-even batch sizes; they complement sampler NFE but do not define hardware-independent model-equivalent costs.

\subsection{A particle-only dual residual estimator}
The negative norm admits a variational form that avoids an explicit score or divergence.  For
$f_t=\nabla\cdot(p_ta_{v_t})$,
\begin{align}
\|f_t\|_{\dot H^{-1}(p_t)}^2
&=\sup_{\phi:\,\mathbb E_{p_t}\phi=0}
\left\{2\langle f_t,\phi\rangle
-\mathbb E_{p_t}\|\nabla\phi\|^2\right\} \\
&=\sup_{\phi}
\left\{-2\mathbb E_{p_t}\langle a_{v_t},\nabla\phi\rangle
-\mathbb E_{p_t}\|\nabla\phi\|^2\right\}.
\label{eq:dual-particle-residual}
\end{align}
The second line follows by integration by parts.  Consequently, samples $x_i\sim p_t$ and evaluations of $a_{v_t}(x_i)$ are sufficient to fit a critic $\phi_\omega$ and estimate the geometric residual without separately estimating $\nabla\log p_t$ or $\nabla\cdot a_v$.  At the population optimum, $\nabla\phi_\star=-P_{\nabla,p_t}a_{v_t}$ and the objective equals $\|P_{\nabla,p_t}a_{v_t}\|_{L^2(p_t)}^2$.  Appendix~\ref{app:unified-new-experiments} validates a finite RBF version of this estimator on the hard-GMM path.  The controlled result supplies direct evidence that the variational route recovers a useful schedule without access to the score.

%% file: appendices/C_mixtures_and_practical_parameterizations.tex
\section{Mixtures and Practical Parameterizations}
\label{app:mixtures}

The Gaussian examples show exact solvability.  For multimodal densities, exact cancellation is harder, but solenoidal feasible directions remain explicit.  Let
\begin{equation}
    p_t(x)=\sum_{m=1}^M\pi_m p_{m,t}(x),
    \qquad
    r_m(x,t)=\frac{\pi_m p_{m,t}(x)}{p_t(x)}.
\end{equation}
If $\nabla\cdot(p_{m,t}u_{m,t})=0$ for every component, then
\begin{equation}
    u_t(x)=\sum_{m=1}^M r_m(x,t)u_{m,t}(x)
    \label{eq:mixture-gauge}
\end{equation}
obeys $\nabla\cdot(p_tu_t)=0$, because
\[
    p_tu_t=\sum_m\pi_m p_{m,t}u_{m,t}.
\]
For Gaussian mixtures in two dimensions, a concrete choice is
\begin{equation}
    u_{m,t}(x)=\omega_m(t)R(x-\mu_m(t)),
    \qquad
    R=\begin{pmatrix}0&-1\\1&0\end{pmatrix}.
\end{equation}
This yields a multimodal gauge that rotates locally around each component center.  It is exactly marginal-preserving in the continuum and can be optimized through the projected objective \eqref{eq:bias-objective}.

For general two-dimensional densities, the constraint can be enforced through a stream function:
\begin{equation}
    p_tu_t=\nabla^\perp\psi_t=(-\partial_2\psi_t,\partial_1\psi_t),
    \qquad
    u_t=\frac{\nabla^\perp\psi_t}{p_t}.
    \label{eq:stream}
\end{equation}
In computations one can use $\psi_t=p_t\chi_t$ or a regularized denominator $p_t+\varepsilon$ to avoid instabilities in low-density regions.  In higher dimensions, one may parameterize $p_tu_t$ as the divergence of a skew-symmetric matrix potential.

%% file: appendices/I_additional_explicit_examples.tex
\section{Additional Explicit Examples}
\label{app:examples}

This appendix records explicit examples used to test and stress the theory.  The goal is to separate five phenomena: disappearance of gauge freedom, exact cancellation, non-Gaussian cancellation, multimodal feasibility, and harmful pure gauges.

\subsection{One-dimensional obstruction}

Let $M=\R$, $p_t(x)>0$, and assume $p_t(x)u_t(x)\to0$ as $|x|\to\infty$.  The solenoidal constraint becomes
\begin{equation}
    \partial_x(p_tu_t)=0 .
\end{equation}
Thus $p_tu_t=c(t)$, and the decay condition forces $c(t)=0$.  Hence $u_t\equiv0$.

\begin{proposition}[No nontrivial decaying gauges in one dimension]
Under the assumptions above, $\cU_p=\{0\}$.  Therefore the optimization problem \eqref{eq:main-opt} is intrinsically multidimensional.
\end{proposition}

\subsection{Conformal anisotropic Gaussian path}

Let
\begin{equation}
    p_t=\mathcal N(0,s(t)^2D),
    \qquad
    D=\diag(a^2,b^2),
    \qquad a,b>0,
\end{equation}
and consider the scaling lift
\begin{equation}
    v_t(x)=\beta(t)x,
    \qquad
    \beta(t)=\frac{\dot s(t)}{s(t)} .
\end{equation}
Let
\begin{equation}
    J=\begin{pmatrix}0&-1\\1&0\end{pmatrix},
    \qquad
    A_t=c(t)JD^{-1},
    \qquad
    u_t(x)=A_tx .
\end{equation}
Since $\Sigma_t=s(t)^2D$, we have
\begin{equation}
    A_t\Sigma_t+\Sigma_tA_t^\top
    =c(t)s(t)^2J-c(t)s(t)^2J=0,
\end{equation}
so $u_t$ is $p_t$-solenoidal.  The corrected lift is $w_t=(\beta I+A_t)x$.  Its acceleration matrix is
\begin{equation}
    M_t=\dot B_t+B_t^2,
    \qquad B_t=\beta I+A_t.
\end{equation}
Using $A_t^2=-c(t)^2(a b)^{-2}I$, we obtain
\begin{equation}
    M_t=
    \left(\frac{\ddot s(t)}{s(t)}-\frac{c(t)^2}{a^2b^2}\right)I
    +(\dot c(t)+2\beta(t)c(t))JD^{-1}.
\end{equation}
The second term is weighted-solenoidal for the conformal Gaussian.  Hence the density-active acceleration vanishes whenever
\begin{equation}
    c(t)^2=a^2b^2\frac{\ddot s(t)}{s(t)} .
\end{equation}

\begin{proposition}[Exact anisotropic cancellation]
If $\ddot s(t)/s(t)>0$, then the conformal anisotropic Gaussian path admits a nonzero linear solenoidal gauge that cancels the leading Euler density-level bias.  The gauge is given by
\begin{equation}
    u_t(x)=ab\sqrt{\frac{\ddot s(t)}{s(t)}}\,JD^{-1}x .
\end{equation}
\end{proposition}

This example is not rotationally isotropic in Euclidean coordinates; the correct rotation is skew only in the covariance-weighted geometry.

\subsection{Radial densities and centrifugal compensation}

Let $p_t(x)=\rho_t(r)$ on $\R^2$, where $r=\|x\|$.  Consider a radial feasible lift
\begin{equation}
    v_t=b_t(r)e_r,
\end{equation}
where $b_t$ satisfies the radial continuity equation
\begin{equation}
    \partial_t\rho_t+\frac1r\partial_r(r\rho_tb_t)=0 .
\end{equation}
Any angular field
\begin{equation}
    u_t=c_t(r)e_\theta
\end{equation}
is $p_t$-solenoidal because
\begin{equation}
    \nabla\cdot(\rho_t c_t(r)e_\theta)=\frac1r\partial_\theta(\rho_t c_t(r))=0 .
\end{equation}
For $w_t=b_t(r)e_r+c_t(r)e_\theta$, the polar acceleration components are
\begin{equation}
    a_r=\partial_tb+b\partial_rb-\frac{c^2}{r},
    \qquad
    a_\theta=\partial_tc+b\partial_rc+\frac{bc}{r} .
\end{equation}
The angular component is density-invisible for radial $p_t$, whereas the radial component is density-active.  Thus the choice
\begin{equation}
    c_t(r)^2=r\left(\partial_tb_t(r)+b_t(r)\partial_rb_t(r)\right)
    \label{eq:radial-c}
\end{equation}
whenever the right-hand side is nonnegative, cancels the density-active acceleration.

\begin{proposition}[Centrifugal cancellation]
For radial density paths in two dimensions, angular solenoidal gauges can cancel the leading Euler density bias whenever \eqref{eq:radial-c} is nonnegative.  The mechanism is centrifugal: the term $-c^2/r$ offsets the material radial acceleration.
\end{proposition}

\subsection{Mixture superposition}

Let
\begin{equation}
    p_t(x)=\sum_{m=1}^M\pi_m p_{m,t}(x),
    \qquad
    r_m(x,t)=\frac{\pi_mp_{m,t}(x)}{p_t(x)}.
\end{equation}
Suppose $u_{m,t}$ is solenoidal with respect to the $m$-th component:
\begin{equation}
    \nabla\cdot(p_{m,t}u_{m,t})=0 .
\end{equation}
Define
\begin{equation}
    u_t(x)=\sum_{m=1}^M r_m(x,t)u_{m,t}(x).
\end{equation}
Then
\begin{equation}
    p_tu_t=\sum_m\pi_mp_{m,t}u_{m,t},
\end{equation}
and therefore
\begin{equation}
    \nabla\cdot(p_tu_t)=0 .
\end{equation}

\begin{proposition}[Mixture of componentwise gauges]
A responsibility-weighted mixture of componentwise solenoidal gauges is exactly solenoidal with respect to the mixture density.  For a two-dimensional Gaussian mixture, a concrete choice is
\begin{equation}
    u_{m,t}(x)=\omega_m(t)R(x-\mu_m(t)),
    \qquad
    R=\begin{pmatrix}0&-1\\1&0\end{pmatrix}.
\end{equation}
\end{proposition}

This gives a feasible multimodal family.  Inside each component, the construction behaves like Gaussian rotation; near mode boundaries, the responsibilities create nontrivial interactions.

\subsection{A harmful pure gauge on the torus}

Let $M=\mathbb T^2$ and $p\equiv1$.  Every divergence-free field preserves the uniform density in the continuum.  Take the stream function
\begin{equation}
    \psi(x,y)=\sin x\sin y,
\end{equation}
and
\begin{equation}
    u=\nabla^\perp\psi=(-\sin x\cos y,\cos x\sin y).
\end{equation}
Then $\nabla\cdot u=0$, so the continuum density is stationary.  However, the Euler density-level defect is governed by $\nabla\cdot a_u$, where $a_u=(u\cdot\nabla)u$.  A direct computation gives
\begin{equation}
    \nabla\cdot a_u = \cos(2x)+\cos(2y),
\end{equation}
In particular, this quantity is not identically zero.

\begin{proposition}[Continuum-invisible does not mean numerically harmless]
There exist divergence-free fields on $\mathbb T^2$ that preserve the uniform density exactly in the continuum but produce nonzero leading Euler marginal bias.  Therefore marginal-preserving gauges must be selected by a solver-aware criterion, not arbitrarily.
\end{proposition}

\subsection{Curvature-distractor Gaussian gauge}
\label{ex:curvature-distractor}

This example makes precise why the full material-acceleration norm is not the same object as marginal density bias.  Let $d=4$ and
\[
    p_t=\mathcal N(0,s(t)^2I_4),\qquad s(t)=1+t^2 .
\]
Let $J_0=\operatorname{diag}(J_2,J_2)$, where
\[
    J_2=\begin{pmatrix}0&-1\\1&0\end{pmatrix},
\]
and let $Q(t)\in O(4)$ rotate the two complex blocks.  Define
\[
    J(t)=Q(t)J_0Q(t)^\top,
    \qquad
    B_t=\beta(t)I+\sqrt{\kappa(t)}J(t),
    \qquad
    \beta(t)=\frac{\dot s(t)}{s(t)},
    \quad
    \kappa(t)=\frac{\ddot s(t)}{s(t)}.
\]
Since $J(t)^\top=-J(t)$ and $J(t)^2=-I$, the skew part $A_t=\sqrt{\kappa(t)}J(t)$ is $p_t$-solenoidal and satisfies
\[
    A_t^2=-\kappa(t)I.
\]
Therefore the Gaussian second-order matching condition from Theorem~\ref{thm:gaussian-main} holds:
\[
    2B_t\Sigma_tB_t^\top=\ddot\Sigma_t,
    \qquad \Sigma_t=s(t)^2I.
\]
Thus the leading Euler covariance bias vanishes, equivalently the density-active residual is zero.  However,
\[
    a_w(x,t)=(\dot B_t+B_t^2)x
\]
contains the term $\sqrt{\kappa(t)}\dot J(t)x$.  By choosing $Q(t)$ to vary rapidly on a short interval, $\|a_w\|_{L^2(p_t)}$ and full material-acceleration proxies can be made arbitrarily large while the leading density-active residual remains zero.

\begin{proposition}[Curvature can be marginal-invisible]
There exist smooth Gaussian lifts $w_t(x)=B_tx$ transporting the same isotropic covariance path such that the leading Euler marginal covariance bias is zero for all $t$, while the full material-acceleration energy $\mathbb E_{p_t}\|a_w\|^2$ is arbitrarily large on a prescribed small time interval.
\end{proposition}

This construction is the analytical basis of the curvature-distractor check summarized in Appendix~\ref{app:synthetic-experiments}.  It shows that scheduling by full material acceleration can chase motion inside a gauge fiber rather than motion of the marginal density itself.

%% file: appendices/J_connection_viewpoint.tex
\section{Connection Viewpoint}
\label{app:connection}

Although the main paper is phrased as constrained optimization, the same construction has a geometric interpretation.  Let $\operatorname{Diff}(M)$ act on densities by push-forward and define
\[
\pi:\operatorname{Diff}(M)\to\mathcal P(M),
\qquad
\pi(\varphi)=\varphi_\#p_0.
\]
The fiber over $p$ consists of all particle maps producing the same density.  The vertical tangent space at $p$ is precisely
\[
\mathcal V_p=\{u:\nabla\cdot(pu)=0\}.
\]
Classical optimal transport chooses a horizontal lift by minimizing kinetic energy, yielding gradient velocities.  In contrast, the present paper chooses a horizontal representative by minimizing modified-equation density bias.

\begin{proposition}[Numerical horizontal condition]
A feasible lift $v$ is stationary for the modified-equation bias functional inside its gauge fiber iff
\[
P_{\mathcal V_p}\left[-D_t^v b_v+(\nabla v)^\top b_v\right]=0.
\]
Thus the solver induces a horizontal distribution distinct from the Otto/optimal-transport horizontal distribution.
\end{proposition}

For a two-parameter density surface $p(s,t)$ with numerical horizontal lifts $v_s,v_t$, the curvature of this solver-induced connection is the vertical component of the commutator
\[
\Omega_p(\partial_sp,\partial_tp)
=
P_{\mathcal V_p}\left(\partial_sv_t-\partial_tv_s+[v_s,v_t]\right).
\]
Nonzero curvature means that two density paths with the same endpoints may induce different particle rearrangements inside the final density fiber.  This offers a geometric interpretation of hidden diversity and holonomy effects, but it is not needed for the projected optimization algorithm.

\paragraph{Relation to Wasserstein splines.}
Second-order optimal-transport models and Wasserstein splines minimize an acceleration action while selecting the marginal curve subject to interpolation constraints \citep{BenamouGallouetVialard2019}.  Here the entire curve $p_t$ is fixed.  The optimization acts vertically, over the family of Lagrangian lifts that realize that curve, and the action is induced by the leading defect of a specified numerical solver.  Thus $\cB$ is not the intrinsic spline energy of the marginal curve; it is a solver-induced second-order transport action on its gauge fiber.